\documentclass{article}
\usepackage{courier}
\RequirePackage[OT1]{fontenc}
\RequirePackage{amsmath}

\usepackage{amsthm,esint}
\usepackage[numbers]{natbib}

\usepackage{bbm}
\usepackage{booktabs}
\usepackage{graphicx}
\usepackage[titletoc,title]{appendix}
\usepackage[letterpaper,left=3.81cm,top=2.54cm,bottom=2.54cm,right=2.54cm]{geometry}
\usepackage[nottoc]{tocbibind}

\usepackage{xparse}

\usepackage{appendix}
\usepackage{chngcntr}
\usepackage{etoolbox}

\usepackage{graphicx}
\usepackage{sidecap}
\usepackage{float}
\usepackage{enumerate,mdwlist}
\usepackage{epsfig}
\usepackage[bbgreekl]{mathbbol}
\usepackage{amsfonts,amsmath,enumerate}
\usepackage{amscd}
\usepackage{amsthm}
\usepackage{mleftright}
\DeclareSymbolFontAlphabet{\mathbb}{AMSb}
\DeclareSymbolFontAlphabet{\mathbbl}{bbold}
\usepackage{colortbl}

\usepackage{amsmath,amsfonts,amssymb,mathrsfs}\usepackage{bm}
\usepackage{latexsym}
\usepackage{mleftright}
\usepackage{amssymb}
\usepackage{color}
\usepackage{xifthen}

\usepackage[titletoc]{appendix}

\usepackage{footnote}

\newlength{\oldparindent}
\usepackage{hyperref} 
\hypersetup{
	colorlinks = true,
	linkcolor = blue,
	anchorcolor = blue,
	citecolor = blue,
	filecolor = blue,
	urlcolor = blue
}

\usepackage[ruled,vlined,resetcount,algosection]{algorithm2e}
\usepackage{makeidx}
\usepackage{hyperref}
\usepackage{amsmath}
\usepackage{amsthm}
\usepackage{amsfonts}
\usepackage{amssymb}
\usepackage{mathrsfs}
\usepackage{graphicx}

\usepackage{amssymb,amsmath,amsthm,nccmath}
\usepackage{mathtools}
\usepackage{hyperref}
\usepackage{enumerate}
\usepackage{dsfont}\usepackage{accents}\usepackage{bigstrut}
\usepackage{multirow}
\usepackage{scalerel}

\usepackage{footnote}
\usepackage{etoolbox}

\newcounter{algoline}

\AtBeginEnvironment{algorithm}{\setcounter{algoline}{0}}

\usepackage[utf8]{inputenc} 
\usepackage[T1]{fontenc}    
\usepackage{hyperref}       
\usepackage{url}            
\usepackage{booktabs}       
\usepackage{amsfonts}       
\usepackage{nicefrac}       
\usepackage{microtype}      

\usepackage{inconsolata}
\usepackage{mathptmx}\usepackage{hyperref}       
\usepackage{url}            
\usepackage{booktabs}       
\usepackage{amsfonts}       
\usepackage{nicefrac}       
\usepackage{microtype}      
\usepackage{amssymb,amsmath,amsthm,nccmath}
\usepackage{mathtools}
\usepackage{graphicx}
\usepackage{hyperref}
\usepackage{enumerate}
\usepackage{dsfont}\usepackage{accents}\usepackage{bigstrut}
\usepackage{multirow}
\usepackage{scalerel}

\usepackage{subfig}
\usepackage{hyphenat}

\renewcommand{\coprod}{\mathbin{\rotatebox[origin=c]{180}{$\prod$}}}

\newcommand{\rr}{{\mathbb{R}}}

\newcommand{\nn}{{\mathbb{N}}}

\newcommand{\fff}{{\mathcal{F}}}

\newcommand{\rrflex}[1]{{\ensuremath{\rr^{#1}
}}}
\newcommand{\rrD}{{\rrflex{D}}}
\newcommand{\rrd}{{\rrflex{d}}}

\newcommand{\rrn}{{\rrflex{n}}}

\NewDocumentCommand\argmin{o}{{\operatorname{argmin}\IfValueT{#1}{_{{#1}}}}}
\NewDocumentCommand\AF{o}{\operatorname{AF}\IfValueT{#1}{
		{
			\left({#1}\right)
		}
}}

\NewDocumentCommand\NNtrunc{o}{{
		\left \lceil{
			\mathcal{NN}^{\sigma}
			\IfValueT{#1}{_{#1}}
		}\right \rceil 
}}

\NewDocumentCommand\NNshal{oo}{
	{nn%
		_{\IfValueT{#1}{#1}}	
		\IfValueF{#2}{^{\fff}}\IfValueT{#2}{^{#2}}
	}
}
\NewDocumentCommand\NNho{oo}{
	{\mathcal{HNN}%
		_{R\IfValueT{#1}{#1}}	
		\IfValueF{#2}{^{\fff}}\IfValueT{#2}{^{#2}}
	}
}

\NewDocumentCommand\NNaff{oo}{
	{\mathcal{NN}%
		_{R,\IfValueT{#1}{#1}}	
		\IfValueF{#2}{^{a}}
	}
}

\NewDocumentCommand{\prodd}{oo}{
	\overset{{#2}}{
		\underset{{#1}}{
			\circlearrowleft
		}
	}
}

\NewDocumentCommand{\NN}{oo}{
	{\mathcal{NN}\IfValueT{#2}{^{#2}}\IfValueF{#2}{^{\sigma}}
		\IfValueF{#1}{_{d,D}}
		\IfValueT{#1}{_{{{{#1}}}}}
	}
}

\NewDocumentCommand{\intt}{o}{{\operatorname{int}
		\IfValueT{#1}{\left({#1}\right)}
}}
\NewDocumentCommand\rsupp{mo}{{\operatorname{supp}
		\IfValueT{#1}{\left({#1}
			\middle|
			\IfValueT{#2}{{{#2}}}\IfValueF{#2}{
				{
					K_{\cdot}
				}
			}
			\right)}
}}

\newtheorem{problem}{Problem}[section]
\newtheorem{defn}{Definition}[section]

\newtheorem{ass}[defn]{Assumption}

\newtheorem{prop}[defn]{Proposition}%
\newtheorem{cor}[defn]{Corollary}%
\newtheorem{lem}[defn]{Lemma}%
\newtheorem{ex}[defn]{Example}%
\newtheorem{thrm}[defn]{Theorem}%
\newtheorem{rremark}[defn]{Remark}%
\newtheorem*{ass*}{Assumption}
\newtheorem*{thrm*}{Theorem}
\newtheorem*{cor*}{Corollary}
\newtheorem*{prop*}{Proposition}

\NewDocumentCommand{\loc}{mo}{
	{
		\left\lceil\left \lceil{{#1}}\right \rceil\right \rceil_{\IfValueF{#2}{\epsilon}\IfValueT{#2}{{#2}}} 
	}
}
\NewDocumentCommand{\tope}{mo}{
	{
	{#1}\mbox{-}\operatorname{tope}
	\IfValueT{#2}{{_{#2}}}
	}
}
\NewDocumentCommand{\locFFNs}{oo}{{
		\loc{
			\mathcal{NN}^{\sigma
			}_{d,D,\IfValueF{#2}{k}\IfValueT{#2}{{#2}}
			}	
		}[\IfValueF{#1}{\epsilon}\IfValueT{#1}{{#1}}]
}}
\NewDocumentCommand{\BLNsNN}{oo}{{
		\tope{
			\mathcal{NN}^{\sigma
			}_{d,D\IfValueT{#2}{,{#2}}
			}	
		}[\IfValueF{#1}{\epsilon}\IfValueT{#1}{{#1}}]
}}
\usepackage{xfrac}
\usepackage{longtable}

\makeatletter
\newcommand{\colim@}[2]{%
  \vtop{\m@th\ialign{##\cr
    \hfil$#1\operator@font colim$\hfil\cr
    \noalign{\nointerlineskip\kern1.5\ex@}#2\cr
    \noalign{\nointerlineskip\kern-\ex@}\cr}}%
}
\renewcommand{\varinjlim}{%
  \mathop{\mathpalette\varlim@{\rightarrowfill@\scriptscriptstyle}}\nmlimits@
}

\title{A Canonical Transform for Strengthening the Local $L^p$-Type Universal Approximation Property}

\author{Anastasis Kratsios
	\thanks{
		Department of Mathematics, Eidgen\"{o}ssische Technische Hochschule Z\"{u}rich, HG G 32.3, R\"{a}mistrasse 101, 8092 Z\"{u}rich.  
		email: \textit{anastasis.kratsios@math.ethz.ch}
	}%
	\and
	Behnoosh Zamanlooy
	\thanks{
		Department of Informatics, Computation, and Economics, University of Z\"{u}rich, Binzm\"{u}hlestrasse 14, 8050 Z\"{u}rich.
				email: \textit{bzamanlooy@ifi.uzh.ch}
	}%
}

\begin{document}






\maketitle

\begin{abstract}
Most $L^p$-type universal approximation theorems guarantee that a given machine learning model class $\mathscr{F}\subseteq C(\mathbb{R}^d,\mathbb{R}^D)$ is dense in $L^p_{\mu}(\mathbb{R}^d,\mathbb{R}^D)$ for any suitable finite Borel measure $\mu$ on $\mathbb{R}^d$.  Unfortunately, this means that the model's approximation quality can rapidly degenerate outside some compact subset of $\mathbb{R}^d$, as any such measure is largely concentrated on some bounded subset of $\mathbb{R}^d$.  This paper proposes a generic solution to this approximation theoretic problem by introducing a canonical transformation which "upgrades $\mathscr{F}$'s approximation property" in the following sense.  The transformed model class, denoted by $\mathscr{F}\text{-tope}$, is shown to be dense in $L^p_{\mu,\text{strict}}(\mathbb{R}^d,\mathbb{R}^D)$ which is a topological space whose elements are locally $p$-integrable functions and whose topology is much finer than usual norm topology on $L^p_{\mu}(\mathbb{R}^d,\mathbb{R}^D)$; here $\mu$ is any suitable $\sigma$-finite Borel measure $\mu$ on $\mathbb{R}^d$.  Next, we show that if $\mathscr{F}$ is any family of analytic functions then there is always a strict "gap" between $\mathscr{F}\text{-tope}$'s expressibility and that of $\mathscr{F}$, since we find that $\mathscr{F}$ can never dense in $L^p_{\mu,\text{strict}}(\mathbb{R}^d,\mathbb{R}^D)$.  In the general case, where $\mathscr{F}$ may contain non-analytic functions, we provide an abstract form of these results guaranteeing that there always exists some function space in which $\mathscr{F}\text{-tope}$ is dense but $\mathscr{F}$ is not, while, the converse is never possible.  Applications to feedforward networks, convolutional neural networks, and polynomial bases are explored.  
\end{abstract}



\noindent
{\itshape Keywords:} Universal Approximation, UAP Upgrading, Strict Topologies, Colimits of Topological Spaces.

\noindent
\let\thefootnote\relax\footnotetext{This research was supported by the ETH Z\"{u}rich Foundation.}

\vspace{-2.5em}
\section{Introduction}\label{s_intro}
Since its first conception in \cite{mcculloch1943logical}, digital computing has led to numerous advances in various scientific areas, 
from computer vision and neuroimaging in~\cite{simonyan2014very,moore2019using}, to signal processing in~\cite{lapedes1987nonlinear}, and in an exponentially growing number of areas where complex problems are finding tractable solutions.  
From the theoretical vantage point, the success of these methods lies in the expressibility of neural networks, which was first demonstrated by \cite{Cybenko}, \cite{hornik1990universal}, and \cite{pinkus1999}.  Their results guaranteed that feedforward neural networks with non-polynomial activation function can approximate any continuous function on any given compact subset of the input space.  In the terminology of \cite{kratsios2019UATs}, their results imply the following "$L^p$-type" universal approximation property (UAP).  
%
%
%
\vspace{-.5em}
\begin{defn}[{Locally $L^p$-Universal Approximation Property}]\label{loc_univ_approx}
A (model class) $\fff\subseteq C(\rrd,\rrD)$ is said to have the local $L^p$-UAP if for every finite Borel measure $\nu$ on $\rrd$ which is absolutely continuous with respect to the Lebesgue measure thereon, every $\delta>0$, and every $f\in L^p_{\nu}(\rrd,\rrD)$ there is an $f_{\delta}\in \fff$ satisfying:
\begin{equation}
    \int_{x \in \rrd} \|f(x)-f_{\delta}(x)\|^p d\nu(x) <\delta
.
\label{eq_intro_motive_no_finite}
\end{equation}
\end{defn}

The shortcoming of~\eqref{eq_intro_motive_no_finite} is reflected in the requirement that $\nu$ must be a finite measure in \cite{hornik1991approximation} and compactly-supported in \cite{kidger2020universal}.  The limitation of being finite stems from the fact that any finite measure must be largely concentrated on some bounded set, outside of which its mass rapidly vanishes.   
\begin{defn}[{Global $L^p$-Universal Approximation Property}]\label{defn_global_lp_UAP}
A (model class) $\fff\subseteq C(\rrd,\rrD)$ has the Global $L^p$-UAP if: for every $\sigma$-finite Borel measure $\mu$ on $\rrd$ which is absolutely continuous with respect to the Lebesgue measure thereon, every $\delta>0$, and every $f \in L^p_{\mu}(\rrd,\rrD)$ there is an $f_{\delta}\in \fff$ satisfying:
$$
\int_{x \in \rrd}\|f(x)-f_{\delta}(x)\|^p d\mu(x)<\delta
.
$$
\end{defn}
In \cite{lu2017expressive}, \cite{park2020minimum}, and in \cite{kidger2020universal} the measure $\nu$ in~\eqref{eq_intro_motive_no_finite} can be replaced with the Lebesgue measure if $\sigma$ is the ReLU activation function; i.e. $\sigma=\max\{0,\cdot\}$.  
However, this stronger version of the local $L^p$-UAP is not available for most feedforward networks with the local $L^p$-UAP.  Thus, the approximation quality of a learning model $\fff\subseteq C(\rrd,\rrD)$ may rapidly degrade outside some compact set, if $\fff$ only has the local $L^p$-UAP but not the global $L^p$-UAP.

\noindent The objective of this paper is to provide a solution to this problem. 
\begin{problem}[UAP-Upgrading]\label{prob_UAP_Upgrading}
If a model class $\fff\subseteq C(\rrd,\rrD)$ has the local $L^p$-UAP then how can we canonically build a model class with the global $L^p$-UAP?
\end{problem}

Our approach is designed to be \textit{modular}, in that our results apply to any model class with the local $L^p$-UAP.  Our results apply to: polynomial bases,  kernel regressors with universal kernels (see \cite{MicchelliUniversalKernels}), feedforward neural networks (see \cite{hornik1990universal}), sparse convolutional neural networks (see \cite{UniversalDeepConv}), the NEU-OLS and NEU-DNN models of \cite{kratsios2021neu} and many other learning models.    


Our solution is summarized as follows.  We begin by observing that~\eqref{eq_intro_motive_no_finite} can be leveraged by specializing a subset of learning models $\{f_i\}_{i=1}^n\subseteq \fff$, where each $f_i$ is specialized on a distinct compact subset $K_n$ of the input space $\rrd$.  We then combine each of these "sub-models" into a single learning ensemble model:
\begin{equation}
 \hat{f}\triangleq
 \sum_{i=1}^n \beta_i f_iI_{K_i} + \beta_0 f_0I_{\rrd-\bigcup_{i=1}^n K_i}
,
\label{eq_f_tope_intro}
\end{equation}
where $\beta_1,\dots,\beta_n \in \rr$.  In analogy with polytopes, which are built by cutting and pasting simpler affine sets together, we call any locally-integrable function $\hat{f}$ with representation~\eqref{eq_f_tope_intro} an $\fff$-architope.  The set of all $\fff$-architopes is denoted by $\tope{\fff}$.  We call $\tope{\fff}$ the \textit{architope upgrade of $\fff$} and we will show that it solves the above problem.  

\subsection*{Notation}
We use $\nn_+$ to denote the set of positive integers.  We denote the space of $p$-integrable functions  by $L^p_{\mu}(\rrd,\rrD)$, and the space of local $p$-integrable functions denoted by $L^p_{\mu,\operatorname{loc}}(\rrd,\rrD)$ from $\rrd$ to $\rrD$.  We also make regular use of the \textit{essential support} of a locally-integrable function, which we define now before moving on.  Let $f \in L^p_{\mu,\operatorname{loc}}(\rrd,\rrD)$, the essential support of $\|f\|$ is defined by 
$
\smash{
\operatorname{ess-supp}(\|f\|)
\triangleq
\rrd -
\bigcup
\left\{
U \subseteq \rrd 
: 
\, U  \mbox{ open and } \|f\|(x)=0 \mbox{ $\mu$-a.e. } x \in U
\right\}.
}
$

\subsection*{Contributions}\label{s_intro_contributions}
Our first main result (Theorem~\ref{thrm_bagged_localization_LP}) shows that the architope upgrade of any dense subset $\fff$ of $C(\rrd;\rrD)$ induces a dense subset of a refinement of the topology on $L^p_{\mu,loc}(\rrd,\rrD)$ (for $1\leq p<\infty$), as well as refining norm topology on the subset $L^p_{\mu}(\rrd,\rrD)$ thereof.  We call the analogous universal approximation property the \textit{strict $L^p$-UAP} (for reasons which will become clearer shortly).  Consequentially, (in Corollary~\ref{cor_upgraded_UAT}) this universal approximation theorem and the relationship between the aforementioned topological spaces implies that $\tope{\fff}$ has the \textit{global $L^p$-UAP} whenever $\fff$ has the \textit{local $L^p$-UAP}.  

Our second main result (Theorem~\ref{thrm_non_stone_weirestrass}) gives conditions under which $\fff$ has the local $L^p$-UAP but not the strict $L^p$-UAP.  Note that, by Theorem~\ref{thrm_bagged_localization_LP} $\tope{\fff}$ must have both the local and strict $L^p$-UAPs.

Our last main result, (Theorem~\ref{thrm_Gain}) in an abstract form of the above two results.  It shows that the architope upgrade strictly increases $\fff$'s expressiveness, in the sense that $\tope{\fff}$ always dense in any "function space" in which $\fff$ is dense, but the converse fails.  
%
%
%

\subsection*{Organization of Paper}\label{s_intro_ss_paper_organization}
This paper is organized as follows. Section~\ref{s_prelim} introduces architopes and covers and discusses the involved function spaces as well as their properties.  Section~\ref{s_Main} contains the paper's main results.  Various concrete examples of the architope upgrade applied to popular learning models are explored thereafter in Section~\ref{sss_applications}.  
The paper then concludes in~\ref{s_Conclusion}.  All proofs and additional background material is relegated to the paper's appendix.

\section{Preliminaries}\label{s_prelim}
Throughout the paper, $\fff\subseteq C(\rrd;\rrD)$ represents a non-empty class of functions that have the local $L^p$-UAP .  For example, $\fff$ may denote the set of feedforward networks of depth $1$ as in \cite{hornik1990universal,kidger2020universal}, arbitrary depth ReLU networks as in \cite{lu2017expressive}, deep sparse convolutional networks as in \cite{UniversalDeepConv}, etc.  
We also fix a measure $\mu$ on $\rrd$ satisfying the following.
\begin{ass}[Regularity of the Reference Measure]\label{ass_ref_mes_regularity}
$\mu$ is a $\sigma$-finite Borel measure on $\rrd$ which is absolutely continuous with respect to the $d$-dimensional Lebesgue measure.  
\end{ass}

\begin{defn}[Partition]\label{defn_parition}
Fix a Borel measure $\mu$ on $\rrd$.  A \textit{partition} $\{K_n\}_{n \in \nn^+}$ of the input space $\rrd$, means a collection of compact subsets of $\rrd$ satisfying $\mu\left(\rrd-\bigcup_{n \in \nn^+} K_n\right) =0$ and 
$\mu(K_n\cap K_m)= 0$ 
 for $n\neq m$, $n,m \in \nn^+$.  
\end{defn}
As we will see shortly (Proposition~\ref{prop_choice_free_construction}), our construction and approximation-theoretic results are independent of our choice of a partition.  We nevertheless note that, in practice, partitions can be given exogenously, learned algorithmically (e.g. via a randomized procedure as in \cite{NaorLee2005RandomPartition} or via a semi-supervised procedure \cite{Zamanlooykratsios2021learning}), or if not such method is available then one can simply take the following.
\begin{ex}\label{ex_simple_regular_partition}
Let $\mu$ denote the Lebesgue measure on $\rrd$.  Then, $\{K_n:=[-n-1,n+1]^d-[n,n]^d\}_{n\in \nn_+}$ is partition.
\end{ex}
Until otherwise specified, we fix a partition on $\rrd$.  For every $n \in \nn^+$, let $I_n$ be the indicator function of $K_n$ and $I_{n}^{ +}$ the indicator of the set $\rrd-\bigcup_{i =1}^n K_i$.  The architope of $\fff$ ($\tope{\fff}$) is defined as follows.  
\begin{defn}[The Architope Upgrade]\label{defn_bagged_loc}
	Let $\fff$ be a non-empty subset of $C(\rrd;\rrD)$.  The architope upgrade of $\fff$, denoted by $\tope{\fff}$, is the collection of all ($\mu$-a.e. equivalence classes of) functions $f:\rrd\rightarrow \rrD$ of the form
	\begin{equation}
	f(x) 
	= 
	\sum_{i=1}^{n} \beta_i I^{}_i (x) f_i(x) 
	+ 
	\beta_0f_0(x)I_n^{+}(x)
	\qquad (\forall x \in \rrd)
	,
	\label{eq_defn_bagged_loc}
	\end{equation}
	where $n\in \nn^+$, $f_0,\dots,f_{n} \in \fff$, $\beta_0,\dots,\beta_n \in \rr$, and at-least one of the $\beta_i$ is non-zero.  Any $f\in \tope{\fff}$ is called an ($\fff$)-architope.
\end{defn} 
\begin{rremark}
We follow the usual convention that for any measure $\mu$ on $\rrd$, functions from $\rrd$ to $\rrD$ are identified with their equivalence class of $\mu$-measurable functions which agree on all sets of positive $\mu$-measure.  
\end{rremark}
\begin{rremark}[The Architope Upgrade is Canonical]\label{remark_modificaiton}
The architope upgrade is canonical in the sense that, by construction if $\mathcal{G}\subseteq \fff$ then $\tope{\mathcal{G}}\subseteq \tope{\fff}$.  Thus $\tope{\cdot}$ defines functor on the (poset) category whose objects are (non-empty) subsets of $C(\rrd,\rrD)$ and whose morphisms are inclusions to the (poset) category whose objects are subsets of $L^p_{\mu,\text{loc}}(\rrd,\rrD)$ and whose morphisms are inclusions.
\end{rremark}
Throughout this paper, we assume that no $K_i$ is redundant or of excessively large measure.
\begin{ass}\label{ass_growth_condition_sequence}
	For each $n \in \nn^+$, the set $K_n$ satisfies $
	0< \mu\left(K_n\right)< \infty.
	$
\end{ass}
\subsection{{Strict $L^p$-Universality}}\label{ss_Composit_space}

In the following, we obtain a stronger conclusion to our "UAP upgrading problem".  For this, we introduce an even stronger universal approximation property than the global $L^p$-UAP, which we call the \textit{strict $L^p$-UAP}, and we show that if $\fff$ has the \textit{local $L^p$-UAP} then $\tope{\fff}$ has the \textit{strict $L^p$-UAP} and in particular the \textit{global $L^p$-UAP}.  

For this, we introduce a new finer topology on the set of locally $p$-integrable function $L^p_{\mu}(\rrd,\rrD)$.  The construction of this topology, and the examination of its elementary properties is the main focus of this preliminary section.  We illustrate the strength of this UAP, with our last motivational result, by showing that polynomials in $C(\rrd,\rrD)$ fail to have the \textit{strict $L^p$-UAP}.  

The terminology "strict" is in analogy with the strict topologies introduced by \cite{buck1958boundedStrictTopologyCB_OGPaper} on the space of bounded continuous functions, which has seen significant attention over the years; e.g. in \cite{MR146645StrcitTopologyCbOLPapers}, and and in recent times \cite{MR3624573StrictTopCb}, and in \cite{MR4217072StrictTopCb}.  This is because our construction is analogous to theirs but using locally integrable functions.  

\subsubsection{{The Topological Space $L^p_{\mu,\text{strict}}(\rrd,\rrD)$}}
\label{ss_prelim_sss_introfineLp}
We build a new, finer topology on the set $L^p_{\mu,\text{loc}}(\rrd,\rrD)$ in a series of steps.  
First, for each $n \in \nn^+$, we define the $L^p$-space of composite patterns with at-most $n$-sub-patterns, denoted by $L^p_{\mu:n}(\rrd,\rrD)$, to be the completion of normed space, whose elements are $f \in L^p_{\mu,\operatorname{loc}}(\rrd,\rrD)$ with $\operatorname{ess-supp}(\|f\|)\subseteq \cup_{i=1}^n K_i$ and normed by
$$
\|f\|_{p:n}\triangleq \max_{i=1,\dots,n} \left(
\int_{x \in \rrd} \|f(x)I_{K_i}(x)\|^p d\mu(x)
\right)^{\frac1{p}}
.
$$
Next, these spaces are then aggregated into $L^p_{\mu:\infty}(\rrd,\rrD)\triangleq \bigcup_{n \in \nn^+ } L^p_{\mu:n}(\rrd,\rrD)$ and equipped with the \textit{finest} topology making each $L^p_{\mu:n}(\rrd,\rrD)$ into a subspace; denote this topology by $\tau_{\infty}$.  

\begin{rremark}[Other "Strict Topologies" on {$L^p_{\mu:\infty}(\rrd,\rrD)$}]\label{remark_strict_comparison}
If we had imposed that the topology on $L^p_{\mu:\infty}(\rrd,\rrD)$ be the finest topology containing each $L^p_{\mu:n}(\rrd,\rrD)$ as subspaces and for which the resulting space is locally-convex then would arrive at the $L^p$-type "strict topology" developed and studied in \cite{MR2784782L1LCSCategory}.  We emphasize that $\tau_{\infty}$ is no courser that topology and therefore any density result with respect to $\tau_{\infty}$ implies density for the "strict topology" of \cite{MR2784782L1LCSCategory}. 
\end{rremark}

Not every function is essentially compactly supported.  Thus, we glue the space $L^p_{\mu:\infty}(\rrd,\rrD)$ with the usual $L^p$ and local $L^p$ spaces.  This "glued space" is the main topological tool throughout this paper's analysis.  

In the following, the norm topology on $L^p_{\mu}(\rrd,\rrD)$ is denoted by $\tau_{L^p_{\mu}}$.  Likewise, the usual Fr\'{e}chet topology on $L^p_{\mu,\operatorname{loc}}(\rrd,\rrD)$ is denoted by $\tau_{L^p_{\mu,\operatorname{loc}}}$.  
\begin{defn}[Strict $L^p$-Spaces]\label{defn_Composited}
$L^p_{\mu,\text{strict}}(\rrd,\rrD)$ is the topological space whose underlying set is $L^p_{\mu,\operatorname{loc}}(\rrd,\rrD)$ and is equipped with the smallest topology containing 
$%
\tau_{\infty} \cup \tau_{L^p_{\mu}} \cup \tau_{L^p_{\mu,\operatorname{loc}}}.
$ 
\end{defn}
Before proceeding, let us note that the topology on $L^p_{\mu,\text{strict}}(\rrd,\rrD)$ is indeed well-defined.  
\begin{prop}[Strict $L^p$-Spaces are Well-Defined]\label{prop_existence}
Let $\{K_n\}_{n=1}^{\infty}$ satisfy Assumption~\ref{ass_growth_condition_sequence}.  Then, the topology $\tau_{\infty}$ exists and, in particular, the smallest topology on $L^p_{\mu,\operatorname{loc}}(\rrd,\rrD)$ containing the open sets in $\tau_{\infty} \cup \tau_{L^p_{\mu}} \cup \tau_{L^p_{\mu,\operatorname{loc}}}$ exists.  
\end{prop}
What is perhaps more surprising, is the fact that the topology on $L^p_{\mu,\text{strict}}(\rrd,\rrD)$ does not depend on our choice of subsets $\{K_n\}_{n=1}^{\infty}$ used to define it.  In fact, any choice of subsets satisfying Assumption~\ref{ass_growth_condition_sequence} defines the same topological space, when following the above construction.  
\begin{prop}[{$L^p_{\mu,\text{strict}}(\rrd,\rrD)$ is Independent of the Choice of a Partition $\{K_n\}_{n=1}^{\infty}$}]\label{prop_choice_free_construction}
Let $\{K_n'\}_{n=1}^{\infty}$ be another partition satisfying Assumption~\ref{ass_growth_condition_sequence} and let $L^p_{\mu,\text{strict}}(\rrd,\rrD)'$ be the strict $L^p$ spaced defined by it.  There is a continuous bijection:
$$
\Psi: L^p_{\mu,\text{strict}}(\rrd,\rrD) \rightarrow L^p_{\mu,\text{strict}}(\rrd,\rrD)'
.
$$
Furthermore, the inverse $\Psi^{-1}$ of $\Psi$ is continuous.  
\end{prop}
\subsubsection{{Properties of $L^p_{\mu,\text{strict}}(\rrd,\rrD)$}}\label{ss_prelim_sss_Lpfine_prop}
We understand what it means to be dense and to converge in $L^p_{\mu,\text{strict}}(\rrd,\rrD)$ by examining the elementary properties of this space, which we do in a series of propositions.  

We find that the space $L^p_{\mu,\text{strict}}(\rrd,\rrD)$ is \textit{Hausdorff}; which means that any singleton in $L^p_{\mu,\text{strict}}(\rrd,\rrD)$ is closed.  From the approximation-theoretic lens, this means that the sequences in $L^p_{\mu,\text{strict}}(\rrd,\rrD)$ have a \textit{unique limit}; i.e., no convergent sequence of models can simultaneously approximate two distinct elements of $L^p_{\mu,\text{strict}}(\rrd,\rrD)$ (see \citep[Chapter III.1]{NagataTopologyGen} for details on Hausdorff topological spaces).  
\begin{prop}\label{prop_Hausdorff}
The space $L^p_{\mu,\text{strict}}(\rrd,\rrD)$ is Hausdorff.  
\end{prop}
Convergence in our space is strictly less common than convergence in $L^p_{\mu,\text{loc}}(\rrd,\rrD)$ and in $L^p_{\mu}(\rrd,\rrD)$.  This is because finer topologies have fewer convergent subsets (see \citep[Chapter 2.4]{NagataTopologyGen}).  
\begin{prop}[{The topology on $L^p_{\mu,\text{strict}}(\rrd,\rrD)$ is Fine}]\label{prop_fineLP_is_fine}
The following holds for the topology on $L^p_{\mu,\text{strict}}(\rrd,\rrD)$:\hfill
\begin{enumerate}
    \item[(i)] The topology on $L^p_{\mu,\text{strict}}(\rrd,\rrD)$ is strictly finer than $\tau_{L^p_{\mu,\operatorname{loc}}}$.  
    \item[(ii)] The subspace topology on $L^p_{\mu,\text{strict}}(\rrd,\rrD)\cap L^p_{\mu}(\rrd,\rrD)$ is strictly finer than the topology $\tau_{L^p_{\mu}}$ on $L^p_{\mu}(\rrd,\rrd)$.
\end{enumerate}
\end{prop}
Thus far, we have described the properties of the topology on $L^p_{\mu,\text{strict}}(\rrd,\rrD)$ but we have not yet provided a tool to help us decide if a model class is dense in $L^p_{\mu,\text{strict}}(\rrd,\rrD)$ or if it is not.  Our next result is precisely this tool.  


Proposition~\ref{prop_fineLP_is_fine} implies that the topology on $L^p_{\mu,\text{strict}}(\rrd,\rrD)$ is strictly finer than the topology on $L^p_{\mu,loc}(\rrd,\rrD)$ and that the subspace topology on $L^p_{\mu,\text{strict}}\cap L^p_{\mu}(\rrd,\rrD)$ is strictly finer than the usual norm topology on $L^p_{\mu}(\rrd,\rrD)$.  Therefore, fewer models are dense in $L^p_{\mu,\text{strict}}(\rrd,\rrD)$ than in $L^p_{\mu,loc}(\rrd,\rrD)$ and $L^p_{\mu}(\rrd,\rrD)$; for their respective topologies.  Hence, in the following, we show that the architope upgrade of a learning model with the local $L^p$-UAP actually has the following UAP, which is strictly stronger than global $L^p$-UAP.
\begin{defn}[{Strict $L^p$-Universal Approximation Property}]\label{defn_fine_lp_UAP}
A (model class) $\fff\subseteq C(\rrd,\rrD)$ has the Global $L^p$-UAP if: for every $\sigma$-finite Borel measure $\mu$ on $\rrd$ which is absolutely continuous with respect to the Lebesgue measure thereon, every $\delta>0$, and every $f \in L^p_{\mu}(\rrd,\rrD)$, $\fff$ is dense in $L^p_{\mu,\text{strict}}(\rrd,\rrD)$.  
\end{defn}
The next result shows that convergence in $L^p_{\mu,\text{strict}}(\rrd,\rrD)$ is a constrained version of convergence in the classical $L^p$-sense, where, the additional requirement which must be satisfied is that any approximating sequence needs to correctly match the set $K_n$ in $\{K_n\}_{n=1}^{\infty}$ containing the essential support of the target function.  This constraint must be satisfied, in addition to the usual $L^p$-type norm convergence.  
\begin{prop}[{Convergence in $L^p_{\mu,\text{strict}}(\rrd,\rrD)$}]\label{prop_identification}
A sequence $\{f_k\}_{k \in \nn^+}$ converges to some $f \in L^p_{\mu:n}(\rrd,\rrD)$ in $L^p_{\mu,\text{strict}}(\rrd,\rrD)$ only if all but a finite number of $f_k$ are in $L^p_{\mu:n}(\rrd,\rrD)$.
\end{prop}
\section{Main Results}\label{s_Main} 
The main theoretical results are now presented.  Our main objective will be to confirm that the architope upgrade solves the following strict version of Problem~\ref{prob_UAP_Upgrading}.  
\begin{problem}[Strict UAP-Upgrading]\label{prob_UAP_Upgradting_strict_version}
If a model class $\fff\subseteq C(\rrd,\rrD)$ has the local $L^p$-UAP then how can we canonically build a model class with the strict $L^p$-UAP?
\end{problem}
Propositions~\ref{prop_fineLP_is_fine} implies that any solution to Problem~\ref{prob_UAP_Upgradting_strict_version} is a solution to Problem~\ref{prob_UAP_Upgrading}.  However, Theorem~\ref{thrm_non_stone_weirestrass}, guarantees that the converse need hold true.  
\subsection{{Architopes are Strict $L^p$-Universal}}\label{s_Main_ss_improvements}
The architope upgrade corrects the shortcoming of classical universal approximators by refining their approximation capabilities.  Indeed, any model class which is universal in the sense of \cite{kidger2020universal} maps to a universal approximator in $L^1_{\mu,\text{strict}}(\rrd,\rrD)$ via the architope upgrade.  
\begin{thrm}[Architope Upgrade are Universal in Strict $L^p$]\label{thrm_bagged_localization_LP}
	Fix $p \in [1,\infty)$, $\mu$ satisfies Assumption~\ref{ass_ref_mes_regularity}, $\{K_n\}_{n \in\nn^+} $ satisfies Assumption~\ref{ass_growth_condition_sequence}, and $\fff$ is locally $L^p$-Universal.  
	Then $\tope{\fff}$ is strictly $L^p_{\mu}$-Universal.  
\end{thrm}
Together, Theorem~\ref{thrm_bagged_localization_LP} and Proposition~\ref{prop_fineLP_is_fine} imply that the $\tope{\fff}$ solves our UAP-Upgrading problem.
\begin{cor}[{${\tope{\fff}}$ is Globally $L^p$-Universal}]\label{cor_upgraded_UAT}
If $\fff$ has the local $L^p$-UAP then $\tope{\fff}$ has the global $L^p$-UAP.  
\end{cor}
Theorem~\ref{thrm_bagged_localization_LP} and Corollary~\ref{cor_upgraded_UAT} guarantee that $\tope{\fff}$ always has the strengthened form of the local $L^p$-UAP, formalized by the strict $L^p$-UAP and the global $L^p$-UAP.  Our next result describes a variety of situations in which there is a "gap" between $L^p$-type UAPs of $\fff$ and its architope upgrade $\tope{\fff}$.  Thus, the next theorem shows that there are families of functions which have the local $L^p$-universal but fail to have the strict $L^p$-UAP.  
\subsubsection{{The Gap Between Local $L^p$-UAP and Strict $L^p$-UAP}}\label{s_main_ss_Lpgaps}
We quantify "the gap between local $L^p$-universality and strict $L^p$-universality", by illustrating that even classical well-behaved approximation-theoretic models, namely polynomials, fail to have the strict $L^p$-UAP.  This is implied by the following necessary condition.
\begin{thrm}[{Gaps in $L^p$-UAPs}]\label{thrm_non_stone_weirestrass}
Let $\mu$ be the Lebesgue measure on $\rrd$.  If $\fff\subseteq C(\rrd,\rrD)$ is dense and each $f\in \fff$ is analytic then $\fff$ is not dense in $L^p_{\mu,\text{strict}}(\rrd,\rrD)$.  Moreover, $\fff$ does not have the strict $L^p$-UAP property.  
\end{thrm}
Theorem~\ref{thrm_non_stone_weirestrass} is general as it applies to any dense analytic family in $C(\rrd,\rrD)$ in the \textit{(uniform convergence on compacts sense) }.  In particular, the result holds for any family satisfying the Stone-Weirestra{\ss}-type conditions of \cite{1994ProllaStoneWeirestrassTheorems}, \cite{Timofte2Liaqad2018JMathAnalApplStoneWeirestrassTheorems}, or of \cite{FalindoSanchis2004StoneWeirestrassTheoremsForGroupValued}.  We illustrate this point with the classical Stone-Weirestra{\ss} setting; i.e., with polynomials.  
\begin{cor}[Polynomials are Not Dense in Strict $L^p$]\label{cor_polynomials_not_dense}
    Let $\fff\left\{\sum_{n=0}^N \beta_n x^n:\, N\in \nn, \beta_n\in \rr,\right\}$.  Then $\fff$ is not strict $L^p$-Universal.  In fact, $\fff$ fails to be dense in $L^p_{\mu,\text{loc}}(\rrd,\rr)$ if $\mu$ is the Lebesgue measure.  
\end{cor}

Together, Theorems~\ref{thrm_bagged_localization_LP} and~\ref{thrm_non_stone_weirestrass} describe a range of situations where there is a "gap" between the expressiveness of $\fff$ and $\tope{\fff}$.  However, there are situations in which the conditions of Theorem~\ref{thrm_non_stone_weirestrass} fail.  
In these cases, even though Theorem~\ref{thrm_bagged_localization_LP} guarantees that $\tope{\fff}$ has the strict $L^p$-UAP, we cannot directly conclude that $\fff$ does not.  

Nevertheless, the following result guarantees that $\tope{\fff}$ is necessarily more expressive than $\fff$.  The result is an abstract form of the above results, showing at least one "function space" in which $\fff$ is not dense and $\tope{\fff}$ is, while simultaneously reassuring us that the converse cannot happen.  
\subsection{Strict Expressibility Improvement for the Architope Upgrade}\label{s_Main_ss_strict_improvement}
The improvement of Theorem~\ref{thrm_bagged_localization_LP} is generally a strict increase in expressiveness.  To formalize this, note that a universal approximation theorem is a statement about the density of certain class of functions in specific function spaces for specific topologies.  Since density is a purely topological property (see \citep[II.3]{NagataTopologyGen}), then comparing the expressiveness of two classes of functions reduces to comparing the topological spaces on which they are dense.  Since distinct topologies on those function spaces emphasise different aspects of the functions therein, then $\tope{\fff}$ is strictly more expressive than $\fff$ only if it is dense whenever $\fff$ is dense and the converse implication fails.   
\begin{thrm}[Architopes Upgrade Strictly Improves Expressiveness]\label{thrm_Gain}
	Let $X$ be a set of functions from $\rrd$ to $\rrD$ containing $\fff$ and $\tope{\fff}$, and let $\sim$ be an equivalence relation on $X$.  Denote the equivalence classes of $\fff$, $\tope{\fff}$, and $X$ under $\sim$ by the same symbols.   If $\fff\neq \tope{\fff}$, then the following hold:
	\begin{enumerate}[(i)]
		\item If $\tau$ is a topology on $X$ making $\fff$ dense then $\tope{\fff}$ is also dense in $X$ for $\tau$.  
		\item There exists a topology on $X$ for which $\tope{\fff}$ is dense but $\fff$ is not.
	\end{enumerate}
\end{thrm}
\subsection{Applications}\label{sss_applications}
We now apply the architope upgrade to various machine learning models.  
\subsubsection{Feedforward Networks}\label{sss_ffNNs}
Let $J$ be a positive integer and $\sigma:\rr \rightarrow \rr$ be a continuous \textit{activation function}. A \textit{feedforward neural network} of depth $J$ from $\rrd$ to $\rrD$ is a continuous function $f$ defined iteratively by
$$
\smash{
\begin{aligned}
&f(x) = W \circ f^{(J)}
,
& \, f^{(j)}(x) = \sigma \bullet \left(
W^{(j)}(f^{(j-1)}(x))
\right),
& \, f^{(0)}(x)=x
,
&\, j=1,\dots,J
,
\end{aligned}
}
$$
where $W^j$ is an affine map from $\rrflex{d_j}$ to $\rrflex{d_{j+1}}$, and $\bullet$ denotes component-wise composition.  
The architecture $\NN[J]$ consists of all feedforward networks of depth at-most $J$.  
\begin{cor}[Architope: Feedforward Case]\label{cor_ffNN_case}
	Let $\sigma$ be a continuous and non-polynomial activation function.  Let $J \in \nn^+$ and $1\leq p<\infty$.  Then:
	\begin{enumerate}[(i)]
	\item $\tope{\NN[J]}$ is dense in $L^p_{\mu,\text{strict}}(\rrd,\rrD)$,
	\item For any $\delta>0$ and $f \in L^p_{\mu}(\rrd,\rr)$ there exists some $f^{\delta}\in \tope{\NN[J]}$ satisfying
	$$
	\int_{x \in \rrd} | f(x)-f^{\delta}(x)|^p d\mu(x) < \delta 
,
	$$
	\item $\tope{\NN[J]}$ is dense in any topology on $L^p_{\mu,\operatorname{loc}}(\rrd,\rrD)$ for which any of $\NN[J],\dots,\NN[1]$ is dense but the converse fails.  
	\end{enumerate}
\end{cor}
\subsubsection{Convolutional Networks}\label{sss_ConvNets}
Let $J,s\in \nn^+$ and $\sigma(x)=\max\{0,x\}$.   
A \textit{convolutional neural network} of depth $J$ and sparsity $s$ is a continuous function $f$ from $\rrd$ to $\rr$ defined iteratively by
$$
\begin{aligned}
&f(x) = W \circ f^{(J)}
,
& \, f^{(j)}(x) = \sigma\bullet\left(
w^{(j)}\star (f^{(j-1)}(x)) - b^j
\right),
& \, f^{(0)}(x)=x
,
&\, j=1,\dots,J
,
\end{aligned}
$$
where $W$ is an affine map from $\rrflex{d + sJ}$ to $\rr$, $b^{(j)}\in \rrflex{d + sj}$, $w^{(j)}=\{w_k^{(j)}\}_{k=-\infty}^{\infty}$ are \textit{convolutional filter masks} where $w_k \in \rr$ and $w_k \neq 0$ only if $0\leq k\leq s$, and the \textit{convolutional operation} of $w^{(j)}$ with the vectors $\{v_j\}_{j=1}^J$ is the sequence defined by
$
(w\star v)_i = \sum_{j=0}^{J-1}
w_{i-j} v_j
$.
The architecture $\operatorname{Conv}^s$ is the set of all convolutional nets from $\rrd$ to $\rr$ of arbitrary depth $J\in \nn^+$ and sparsity $s$.
\begin{cor}[Architope: Deep Sparse Convolutional Networks]\label{cor_cnn_case}
	Fix $2\leq s\leq d$ and $1\leq p<\infty$.  %
	\begin{enumerate}[(i)]
	\item $\tope{\operatorname{Conv}^s}$ is dense in $L^p_{\mu,\text{strict}}(\rrd,\rrD)$,
	\item For any $\delta>0$ and $f \in L^p_{\mu}(\rrd,\rr)$ there exists some $f^{\delta}\in \operatorname{Conv}^s$ such that
	$$
	\int_{x \in \rrd} |f(x)-f_{\delta}(x)|^p d \mu(x) < \infty
	,
	$$ 
	\item $\tope{\operatorname{Conv}^s}$ is dense in any topology on $L^p_{\mu,\operatorname{loc}}(\rrd,\rrD)$ for which $Conv^s$ is dense but the converse fails.  
	\end{enumerate}
\end{cor}
\subsubsection{Polynomial Bases}\label{sss_poly_basises}
As a final example, we apply our main result to obtain a strict version the Weirestra{\ss} Theorem.  
\begin{cor}[{Architope: For Polynomials}]\label{cor_polynomials_dense}
The set \label{cor_polynomials_not_dense}
    Let $\fff\left\{\sum_{n=0}^N \beta_n x^n:\, N\in \nn, \beta_n\in \rr,\right\}$.  Then, $\tope{\fff}$ has the strict $L^p$-UAP and $\fff$ does not.  
\end{cor}
More generally, we have the following joint consequence of Theorem~\ref{thrm_bagged_localization_LP} and Theorem~\ref{thrm_non_stone_weirestrass}.  
\begin{cor}[{Architope: Analytic Local $L^p$-UAP Models}]\label{cor_polynomials_dense_gen}
The set \label{cor_polynomials_not_dense}
    Let $\fff\subseteq C(\rrd,\rrD)$ consist of analytic functions and suppose that $\fff$ has the local $L^p$-UAP.  Then $\tope{\fff}$ has the strict $L^p$-UAP and $\fff$ does not.  
\end{cor}
When juxtaposing the negative result of Corollary~\ref{cor_polynomials_not_dense} against the positive result in Corollary~\ref{cor_polynomials_dense} we see that the architope upgrade can strictly improve expressiveness within the $L^p$-type context.  In this case, Corollary~\ref{cor_polynomials_dense} plays the role of a concrete version of Theorem~\ref{thrm_Gain} since $L^p_{\mu,\text{strict}}(\rrd,\rrD)$ is a topological space in which the polynomials are not dense while their architope upgrade is dense.  
\section{Conclusion}\label{s_Conclusion}
In this paper, we introduce a generic transformation called the \textit{architope upgrade} that canonically builds a model class with the global $L^p-$UAP from classes with the local $L^p$-UAP property .  In Theorem~\ref{thrm_bagged_localization_LP}, we showed that if $\fff$ is universal in $L^p_{\nu}(\rrd,\rrD)$ for any finite Borel measure on $\rrd$ then $\tope{\fff}$ is necessarily universal in $L^p_{\mu,\text{strict}}(\rrd,\rrD)$ for any $\sigma$-finite Borel measure on $\rrd$.  In particular, we canonically solve our motivational Problem~\ref{prob_UAP_Upgradting_strict_version} (and consequentially our motivational Problem~\ref{prob_UAP_Upgrading}).  

Next, Theorem~\ref{thrm_non_stone_weirestrass} showed that if $\fff$ is a family of analytic models which has the local $L^p$-UAP then $\tope{\fff}$ is dense in $L^p_{\mu,\text{strict}}(\rrd,\rrD)$, for the Lebesgue measure $\mu$ on $\rrd$, while $\fff$ is not.  This result illustrated a strict "gap", or improvement, in the approximation capabilities of $\tope{\fff}$ over $\fff$.  

Lastly, in Theorem~\ref{thrm_Gain} presented the following abstraction of the aforementioned results.  Namely, it showed that the architope modification strictly increases the expressibility of any machine learning model.  This is because if $\tope{\fff}$ is dense in any function space in which $\fff$ is dense, but the converse typically fails.  

Concrete examples of the architope upgrade were examined for feedforward networks, deep convolutional networks, and polynomial bases.  We believe that the scope, and simplicity of the architope upgrade, allow it to be immediately used to improve the approximation capabilities of any machine learning model.  

\section{Acknowledgment}
The authors would like to thank Josef Teichmann and his working group at ETH Z\"{u}rich for their support and feedback throughout the project's development.  We would also like to thank the ETH Z\"{u}rich for its funding.  

\appendix
\begin{appendices}
\section*{Appendix}
The appendix is organized as follows.  Additional mathematical background material is covered in Section~\ref{s_Supplementary_Material} and the proofs of the paper's results are found in Section~\ref{s_Proofs}.  
	\section{Background}\label{s_Supplementary_Material}
	\subsection{Spaces of $p$-Integrable Functions}\label{s_Lp_space}
As discussed in the introduction, both from a practical and theoretical perspective, it is necessary to establish the expressibility of neural network architectures for measures which may not be finite on $\rrd$.  However, if one abandons finite measures then many prevalent functions, such as most polynomials, logarithmic, and exponential functions, fail to belong to $L_{\mu}^p(\rrd;\rrD)$.

This issue is overcome by replacing $L_{\mu}^p(\rrd,\rrD)$ with the standard larger space of locally $p$-integrable functions%
, denoted by $L^p_{\mu,\operatorname{loc}}(\rrd;\rrD)$, consisting of all $\mu$-measurable functions for which $\|f\|^p$ is integrable on every non-empty compact
$K\subset \rrd$.  
In this space, a sequence $\{f_n\}_{n \in \nn}$ converges to some $f \in L^p_{\mu,\operatorname{loc}}(\rrd,\rrD)$ if for every non-empty compact $K\subset \rrd$ and every $\epsilon>0$ there exists some $N \in \nn$ such that
$
\int_{x \in K} \|f(x)-f_n(x)\|^p d\mu(x)<\epsilon %
,
$
for every $n \geq N$.

Since all continuous functions are uniformly bounded on compacts then they belong to $L^p_{\mu,\operatorname{loc}}(\rrd,\rrD)$ and in particular this space contains all the aforementioned functions.  
However, $L^p_{\mu,\operatorname{loc}}(\rrd,\rrD)$ is not an appropriate replacement for $L^p_{\mu}(\rrd,\rrD)$, since analogously to the $L^p$ spaces for finite measures, approximation in $L^p_{\mu,\operatorname{loc}}(\rrd,\rrD)$ underestimates errors.  This is because its topology, denoted by $\tau_{L^p_{\mu,\operatorname{loc}}}$, can be described by the following metric
\begin{equation}
d_{L^p_{loc}}(f,g)\triangleq \sum_{n\in \nn} \frac1{2^n} \frac{
	\int_{x \in \rrd} 
	\left\|
	(f(x)-g(x))I_{K_n}
	\right\|^pd\mu(x)
}{
	1
	+ 
	\int_{x \in \rrd} 
	\left\|
	(f(x)-g(x))I_{K_n}
	\right\|^pd\mu(x)
}
\label{eq_locally_integrable_definition_metric}
.
\end{equation}
In contrast, approximation in $L^p_{\mu}(\rrd,\rrD)$ with its usual topology, is described by the metric
\begin{equation}
d_{L^p}(f,g)= \sum_{n\in \nn} \int_{x \in \rrd} \|(f(x)-g(x))I_{K_n}\|^p d\mu(x) < \infty
.
\label{eq_see_brah}
\end{equation}
Unlike~\eqref{eq_see_brah}, the metric of~\eqref{eq_locally_integrable_definition_metric} shrinks approximation errors made on $K_n$ by a factor of at-least $\frac1{2^n}$.  
Let $\tau_{L^p_{\mu}}$ denote the topology on $L^p_{\mu}(\rrd,\rrD)$ induced by this metric.  
It can be shown that $\tau_{L^p_{\mu}}$ is strictly finer than $\tau_{L^p_{\mu,\operatorname{loc}}}$ on $L^p_{\mu}(\rrd,\rrD)$.  Conversely, however, $L^p_{\mu,\operatorname{loc}}(\rrd;\rrD)$ strictly contains $L^p_{\mu}(\rrd,\rrD)$ as a set.  

\subsection{Combining Topological Spaces}\label{ss_Colimits}
\subsubsection{Gluing Topological Spaces}
	The typical example of a large set containing any prescribed collection of sets $X_1,\dots,X_n$ is their Cartesian product $X_1\times\dots\times X_n$.  This is defined by concatenating all ordered pairs of elements in $X_1,\dots, X_n$.  However, there is a more "efficient way" to combine $X_1,\dots,X_n$, this is their \textit{disjoint union}. The disjoint union of $X_1,\dots,X_n$ is constructed by viewing $X_1,\dots,X_n$ as distinct and independent members of the same set defined by
	$$
	\sqcup_{i=1}^n X_i \triangleq \left\{
	(x,i):\, x \in X_i \, i=1,\dots,n
	\right\}.
	$$
	Put another way, $\sqcup_{i=1}^n X_i$ is the smallest set including each $X_i$, distinctly.  
	
	The analogous construction can be made for topological spaces.  The disjoint union of topological spaces $X_1,\dots,X_n$ is the smallest topological space containing distinct copies of each $X_1,\dots,X_n$.  This space is defined as the topological space whose underlying set is the disjoint union of the sets $\sqcup_{i=1}^n X_i $ and its topology is defined as being the \textit{finest topology} ensuring that the inclusions of $X_1,\dots,X_n$ are continuous functions.  
	\begin{ex}\label{ex_R}
		The Cartesian product of $\rr\times \rr =\rr^2$.  However, the disjoint union of $\rr$ and $\rr$ can be identified with distinct two vertical lines in $\rr^2$
		$
		\left\{(x,-1):\, x \in \rr\right\} \cup \left\{(x,1):\, x \in \rr\right\}
		.
		$
	\end{ex}
	       %
%
As illustrated by Example~\ref{ex_R} the disjoint union construction is primitive as it ignores any structure shared by any non-empty collection of topological spaces $\{X_i\}_{i \in I}$ since it artificially amalgamates the spaces.  This artificiality is circumvented in \cite{BourbakiTopGen} through a upgrade of the disjoint union construction where $x,z \in \coprod_{i \in I} X_i$ are identified if $x=z$.  This identification defines an equivalence relation on $\coprod_{i \in I} X_i$, furthermore, the quotient map $q: \coprod_{i \in I} X_i \rightarrow \bigcup_{i \in I} X_i$ is continuous.  Moreover, this topology always exists and is optimal in the following sense.  
	\begin{lem}[Final Topology {\citep[Proposition I.2.4]{BourbakiTopGen}}]\label{thrm_cocompletness_countable}
		For every non-empty collection of topological spaces $\{X_i\}_{i \in I}$ there exists a unique finest topology on $\bigcup_{i \in I} X_i$ making all the inclusion maps $X_n\rightarrow \bigcup_{i \in I} X_i$ into continuous functions.  
	\end{lem}
We will require the following special case of Lemma~\ref{thrm_cocompletness_countable}.  Suppose that $I$ is a partially ordered indexing set, whose partial order we denote by $\leq$.  Suppose moreover, that we are given a family of topological spaces $\{X_i\}_{i \in I}$ indexed by $I$ and a family of continuous maps 
$\{f_{i,j}:X_i \to X_j \mbox{ if } i \leq j\}$ such that if $i =j$ then $f_{i,i}$ is the identity map on $X_i$ and if $i\leq j\leq k$ then 
$f_{i,k} = f_{j,k}\circ f_{i,j}$.  We call such a construction a \textit{direct system} of topological spaces.  For example, $I$ may be taken to be 
$\nn^+$ with usual ordering, $\{X_n\}_{n \in \nn^+}$ is a sequence of sub-spaces of a larger topological space $X$, and $f_{i,j}$ may be taken to be the inclusion functions.  

Then Lemma~\ref{lem_details_cocompleteness_in_Top} can be used to formalize the limit of this direct system of topological spaces, denoted by $\varinjlim_{i \in I} X_i$.  Informally, it is the smallest topological space in which the entire direct system $(\{X_i\}_{i \in I},\{f_{i,j}\})$ can be embedded.  This construction ties into our program via
\begin{equation}
    \varinjlim_{n \in \nn} L^p_{\mu:n}(\rrd,\rrD)
    =
L^p_{\mu:\infty}(\rrd,\rrD)
.
\label{eq_direct_limit_L_infinity}
\end{equation}
The formulation of~\eqref{eq_direct_limit_L_infinity} allows the use of the following tools from algebraic topology and category theory.  These will be integral to establishing Proposition~\ref{prop_choice_free_construction}, namely, that the strict $L^p$ space is defined independently of the chosen partition $\{K_n\}_{n \in \nn^+}$.  In what follows, we use $X\cong Y$ to denote the existence of a homeomorphism between two topological spaces, that is, $X$ and $Y$ are topologically identical.  
\begin{lem}\label{lem_universal_constructions}
Let $(\{X_i\}_{i \in I},\{f_{i,j}\})$ be a direct system (of topological spaces) indexed by a directed set $I$ containing $\nn$ as a directed subset.  
\begin{enumerate}[(i)]
    \item Existence and Description {\citep[Page 5]{SpanierTop1995}}: Then $\varinjlim_{i \in I} X_i$ exists and it is given by $\bigcup_{i \in I} X_i$ equipped with the final topology, of Lemma~\ref{lem_details_cocompleteness_in_Top}, induced by the inclusions $X_i \to \bigcup_{i \in I} X_i$,
    \item Minimality \citep[Tag 002D]{stacks-project}: Let $Y$ be a topological space, such that, for every $i \in I$ there is a continuous function $g_i:X_i\to Y$ satisfying $g_j \circ f_{i,j}= g_i$, for every $i\leq j$, then there exists a continuous unique map $\phi$ satisfying
    $$
    \phi: \varinjlim_{i \in I} X_i \rightarrow Y, \mbox{ such that } g_i\circ \phi\circ \iota_{X_i}
    ,
    $$
    for every $i \in I$, where $\iota_{X_i}: X_i \to \varinjlim_{i \in I} X_i$ is the inclusion map.  In particular, if each $g_i$ is a homeomorphism, then so is $\phi$.  
    \item Cofinal Sublimits {\cite[\href{https://stacks.math.columbia.edu/tag/09WN}{Tag 09WN}]{stacks-project}}: If $\{X_n\}_{n \in \nn^+}\subseteq \{X_i\}_{i\in I}$ is such that for every $i\in I$ there exists some $n_i \in \nn^+$ such that $i\leq k_i$ then 
    $$
    \varinjlim_{i \in I} X_i \cong \varinjlim_{n \in \nn^+} X_n
    .
    $$
    \end{enumerate}
\end{lem}
Property (i) expresses the fact that $\bigcup_{i\in I} X_i$ topologized in the above way, is the smallest topological space containing each $X_i$ as a sub-space, and its topology is the strongest possible topology which has this property.  This property is useful to us, since it gives an explicit description of the direct limit and, in particular, it guarantees the existence of $L^p_{\mu:\infty}(\rrd,\rrD)$ since this spaces is precisely the direct limit of the direct system
$\left(
    \left\{
        L^p_{\mu:n}(\rrd,\rrD)
    \right\}_{n \in \nn^+}, 
    \iota_{n,m}
\right)$ where 
    $
    \iota_{n,m}: 
    L^p_{\mu:n}(\rrd,\rrD) \to L^p_{\mu:m}(\rrd,\rrD)
    $ 
    are the inclusion maps for $n \leq m$.  

Property (ii) expresses the minimality of $\varinjlim_{i \in I} X_i$, since every compatible system of continuous functions which is compatible with the direct system can always be unambiguously summarized by a single continuous map $\phi$ from the direct system's direct limit $\varinjlim_{i \in I} X_i$.  In particular, if each $g_i$ is a homeomorphism, then it is easy to see that so is $\phi$.  This will be important for us when establishing Proposition~\ref{prop_choice_free_construction}.  

Property (iii) states that this space can be equivalently defined by a smallest direct system.  This will prove convenient when showing that the strict $L^p$ spaces are well-defined and defined independently of the choice of partitioning sets $\{K_n\}_{n \in \nn^+}$ satisfying Assumption~\ref{ass_growth_condition_sequence}.  
		\subsubsection{Direct Sums and Certain Direct Limits Involving Banach Spaces}\label{ss_Interpol_spaces}
	The theory of interpolation spaces, treated in \cite{LunardiInterpolationTheory}, was established in order to describe Banach spaces which lie in between other Banach spaces.  These typically concern sums or intersections of Banach sub-spaces of a suitable overarching topological vector space.  In this paper, we only require the following situation.  
	
	Let $X$ be a Fr\'{e}chet space and $\{X_i\}_{i\in \nn}$ be Banach sub-spaces of $X$, where $\|\cdot\|_{(i)}$ is the norm on $X_i$, for $i \in \nn$.  Consider the linear subspace $\overset{n}{\underset{i=1}{\bigoplus}} X_i$ comprised of all sums of the form
	$
	\sum_{i=1}^n x_i$, where $x_i \in X_i.$  
	Following \citep[Page ix]{LunardiInterpolationTheory}, $\overset{n}{\underset{i=1}{\bigoplus}} X_i$ is equipped with the norm $\|\cdot\|_n'$ defined by
	\begin{equation}
	\left\|
	x
	\right\|_{n}' \triangleq 
	\inf
	\left\{
	\sum_{i=1}^n \|x_i\|_{(i)}:\,  x=\sum_{i=1}^n x_i , \, x_i \in X_i
	\right\},
	\label{eq_an_interpolation_norm}
	\end{equation}
	and it defines a subspace of $X$.  The infimum in~\eqref{eq_an_interpolation_norm} is required since the representation of any $x \in \overset{n}{\underset{i=1}{\bigoplus}} X_i$ as a sum of elements in $\{X_i\}_{i=1}^n$ is in general not unique.  
	However, if $X_i \cap X_j =\{0\}$ for $i\neq j$, $i,j \in \nn^+$, then for any $x \in \overset{n}{\underset{i=1}{\bigoplus}} X_i$ there necessarily exists a unique $x_i \in X_i$, for $i=1,\dots,n$, such that $x= \sum_{i=1}^n \|x_i\|_{(i)}$.  
	Therefore, in this situation $\|x\|_n'$ reduces to
	$
	\|x\|_n' = \sum_{i=1}^n \|x_i\|_{(i)}
	.
	$
	The next lemma describes relevant aspects of this construction in more detail.  
	
	We denote the $\ell^p$ norm on $\rrn$, for $p \in [1,\infty]$, by $\|\cdot\|_{\ell^p}$.  When $p \in [1,\infty)$ for any $y=(y_i)_{i=1}^n \in \rrn$ the quantity $\|y\|_{\ell^p}$ is defined by
	$
	\|y\|_{\ell^p}\triangleq\left(\sum_{i=1}^n |y_i|^p\right)^{\frac1{p}},
	$
	and when $p=\infty$ the quantity $\|y\|_{\ell^{p}}$ is defined by
	$
		\|y\|_{\ell^{\infty}}\triangleq\underset{i=1,\dots,n}{\max} |y_i|
	.
	$
	\begin{lem}\label{lem_lininity_Desciption}
		Let $\{X_i\}_{i \in \nn}$ be Banach sub-spaces of a Fr\'{e}chet space $X$ and suppose that $X_i \cap X_j = \{0\}$ if $i \neq j$, for each $i,j \in \nn^+$.  Then, for each $n\leq m \in \nn^+$, the following holds:
		\begin{enumerate}[(i)]
			\item For each $p \in [1,\infty]$, the map $\|\cdot\|_n^{(p)}$ taking any $x \in \bigoplus_{i=1}^{n} X_i$ to the real-number 
			$\left\|\left(\|x_1\|_{(1)},\dots,\|x_n\|_{(n)}\right)\right\|_{\ell^p}$ defines a norm on $\overset{n}{\underset{i=1}{\bigoplus}} X_i$, 
			\item For each $p \in [1,\infty]$, the norms $\|\cdot\|_n^{(p)}$ and $\|\cdot\|_n'$ are equivalent on $\overset{n}{\underset{i=1}{\bigoplus}} X_i$,
			\item The completion of $\overset{n}{\underset{i=1}{\bigoplus}} X_i$ with respect to $\|\cdot\|_n^{(p)}$ (resp. $\|\cdot\|_n'$) coincides with the closure of $\overset{n}{\underset{i=1}{\bigoplus}} X_i$ in $X$,
			\item The completion of $\overset{n}{\underset{i=1}{\bigoplus}} X_i$ with respect to the norm $\|\cdot\|_n^{(p)}$ is contained in the completion of $\overset{m}{\underset{i=1}{\bigoplus}} X_i$ with respect to the norm $\|\cdot\|_m'$.  
		\end{enumerate}
	\end{lem}
\begin{proof}[{Proof of Lemma~\ref{lem_lininity_Desciption}}]
In \citep[Theorem 1]{benavides1992weak}, it is shown that $\|\cdot\|_n^{(p)}$ defines a norm on $\overset{n}{\underset{i=1}{\bigoplus}} X_i$.  This gives (i).  

By \citep[Theorem 3.1]{conway2013course} every norm on $\rrn$ is equivalent, and in particular this is true of the $\ell^1$ and $\ell^{\infty}$, that is, there exists constants $0<c,C$ such that
\begin{equation}
c \|y\|_{\ell^p} \leq \|y\|_{\ell^{1}} \leq C \|y\|_{\ell^p} \qquad (\forall p \in \rrn)
.
\label{eq_equivalent_FDN}
\end{equation}
Since, the norms $\|\cdot\|_n^{(p)}$ and $\|\cdot\|_n'$ on $\bigoplus_{i=1}^{n} X_i$ can be rewritten as
$$
\|x\|_n^{(p)} = 
\left\|
\left(
\|x_i\|_{(1)},\dots,\|x_i\|_{(n)}
\right)
\right\|_{\ell^{p}}
\mbox{ and }
\|x\|_n'=
\left\|
\left(
\|x_i\|_{(1)},\dots,\|x_i\|_{(n)}
\right)
\right\|_{\ell^1}
,
$$
respectively, then they are equivalent by~\eqref{eq_equivalent_FDN}; i.e.:
$$
c \|x\|_{n}^{(p)} \leq \|x\|_{n}' \leq C \|x\|_{n}^{(p)}
,
$$
for all $x \in \overset{n}{\underset{i=1}{\bigoplus}} X_i$, where $c,C$ are as in~\eqref{eq_equivalent_FDN}.  This gives (ii).

Since $\overset{n}{\underset{i=1}{\bigoplus}} X_i$ is a subspace of $X$ then the inclusion map $i:\bigoplus_{i=1}^{n} X_i\to X$ is continuous.  By definition $\overset{n}{\underset{i=1}{\bigoplus}} X_i$ is dense in its closure $\overline{\overset{n}{\underset{i=1}{\bigoplus}} X_i}$ in $X$, and since any continuous function can be uniquely extended from a dense subset to the entire set then norm $\|\cdot\|_n^{(p)}$, for any $p \in [1,\infty]$, can be uniquely be continuously extended to all of $\overline{\overset{n}{\underset{i=1}{\bigoplus}} X_i}$.  Since, by definition, $\overline{\bigoplus_{i=1}^{n} X_i}$ is closed it is complete.  Since $\overline{\overset{n}{\underset{i=1}{\bigoplus}} X_i}$ is a complete normed space it is a Banach space.  Moreover, the universal property of the completion of the normed linear space $\bigoplus_{i=1}^{n} X_i$ implies that it must be (up to linear isometry) a subset of $\overline{\bigoplus_{i=1}^{n}X_i}$.  However, since any complete space is closed and $\overline{\overset{n}{\underset{i=1}{\bigoplus}} X_i}$ is the smallest closed set containing $\bigoplus_{i=1}^{n} X_i$ then it must coincide with the completion of $\bigoplus_{i=1}^{n} X_i$ with respect to the norm $\|\cdot\|_n^{(p)}$, for any $p \in [1,\infty]$.  This gives (iii).  

If $n< m$ then any $x \in \overset{n}{\underset{i=1}{\bigoplus}} X_i$ is represented by $x=\sum_{i=1}^n x_i$ for some unique $x_i \in X_i$ and therefore it is represented uniquely as an element of $\overset{m}{\underset{i=1}{\bigoplus}} X_i$, by $x=\sum_{i=1}^n x_i + \sum_{i=n+1}^m 0$ since $0 \in X_i$ for $i=n+1,\dots,m$.  Thus, (iv) follows from (iii).  
\end{proof}
\begin{lem}\label{lem_details_cocompleteness_in_Top}
	Let $\{X_n\}_{n \in \nn}$ be Banach sub-spaces of a Fr\'{e}chet space $X$, suppose that $X_i \cap X_j = \{0\}$ for every $i,j \in \nn^+$ with $i\neq j$.  Then
	\begin{enumerate}[(i)]
		\item There is a unique (up to homeomorphism) finest topology $\tau$ on 
		$
		\bigcup_{n \in \nn^+} \overline{
			\overset{n}{\underset{i=1}{\bigoplus}} X_i
		}
		$ making each $\overset{n}{\underset{i=1}{\bigoplus}} X_i$ into a subspace,
		\item $\tau$ is strictly finer than the subspace topology induced by restriction from $X$.
	\end{enumerate}
\end{lem}
\begin{proof}[{Proof of Lemma~\ref{lem_details_cocompleteness_in_Top}}]
	By \citep[Proposition 4.5.1]{JarchowLCSs} there exists a unique (up to homeomorphism) finest topology $\tau'$ on $\bigcup_{n \in \nn^+} \overset{n}{\underset{i=1}{\bigoplus}} X_i$ 
	making $\overline{\overset{n}{\underset{i=1}{\bigoplus}} X_i}$ into a linear subspace, for each $n \in \nn^+$, while making $\bigcup_{n \in \nn^+} \overline{ \overset{n}{\underset{i=1}{\bigoplus}} X_i}$ into a locally convex space (see \citep[Chapter 3]{OsborneLCSs2014} for more details on locally convex spaces).  
	
	Since each $\overline{\overset{n}{\underset{i=1}{\bigoplus}} X_i}$ is a 
	linear subspace of $X$ then $\bigcup_{n \in \nn^+} \overline{ \overset{n}{\underset{i=1}{\bigoplus}} X_i}$ can be viewed as a linear subspace of $X$.  Since $X$ is a Fr\'{e}chet space then 
	its topology is metric and therefore the subspace topology on $\bigcup_{n \in \nn^+} \overline{ \overset{n}{\underset{i=1}{\bigoplus}} X_i}$ induced by restriction of $X$'s Fr\'{e}chet topology is also metric.  
	
	Note that $X$ is locally-convex, since it is a Fr\'{e}chet space.  Note also that $\tau'$ makes $\bigcup_{n \in \nn^+} \overline{ \overset{n}{\underset{i=1}{\bigoplus}} X_i}$ into a locally convex space containing each $\overline{\overset{n}{\underset{i=1}{\bigoplus}} X_i}$, for every $n \in \nn^+$, as a linear subspace.  Since $\tau'$ is the finest topology satisfying these two conditions then $\tau'$ is at-least as strict as $\tau''$, where $\tau''$ is the restriction of the Fr\'{e}chet topology of $X$ to the subset $\bigcup_{n \in \nn^+} X_n$; i.e.: $\tau''\subseteq \tau'$.

	However, for each $n \in \nn^+$, $\overline{\overset{n}{\underset{i=1}{\bigoplus}} X_i}$ is a proper Banach subspace of the Banach space $\overline{\overset{n+1}{\underset{i=1}{\bigoplus}} X_i}$ and therefore \citep[Corollary 3]{narayanaswami1986spacesLBSpacesNeverMetrizable} guarantees that $\tau'$ is not metrizable.  
	In contrast, since $X$ is Fr\'{e}chet then its topology is by definition metrizable and in particular (the subspace) topology $\tau''$ is metrizable.  Therefore, $\tau'$ is strictly finer than $\tau''$; i.e.: $\tau''\subset \tau'$.
	
	By Lemma \ref{thrm_cocompletness_countable} there exists a unique (up to homeomorphism) finest topology $\tau$ on $\bigcup_{n \in \nn^+} \overline{ \overset{n}{\underset{i=1}{\bigoplus}} X_i}$ making each $ \overline{ \overset{n}{\underset{i=1}{\bigoplus}} X_i}$ into a subspace of $\bigcup_{n \in \nn^+} \overline{ \overset{n}{\underset{i=1}{\bigoplus}} X_i}$ and since $\tau'$ accomplishes this with the additional constraint that it makes $\bigcup_{n \in \nn^+} \overline{ \overset{n}{\underset{i=1}{\bigoplus}} X_i}$ into a locally-convex space, then $\tau$ is at-least as strict as $\tau'$.  In particular, $\tau''\subset \tau'\subseteq \tau$, therefore $\tau$ is strictly finer than $\tau''$. 
\end{proof}

Next, the proofs of the paper's central results are given.  
\section{Proofs}\label{s_Proofs}
This section of the supplementary material contains the proofs of the paper's results.  
\subsection{Proof from Section $2$}
\subsubsection{Technical Lemmas}\label{ss_Spacing_out}
This sub-section centers around results from Section $2$ and results concerning the construction of $L^p_{\mu,\text{strict}}(\rrd,\rrD)$.  We impose some notation.  For each non-empty compact $K\subset \rrd$, let $L^p_{\mu}(K)$ be the linear subspace of $L^p_{\mu,\operatorname{loc}}(\rrd,\rrD)$ comprised of elements $f$ for which $\operatorname{ess-supp}(\|f\|) \subseteq K$.  The following can be said about $L^p_{\mu}(K_n)$, for any $K_n$ in $\{K_n\}_{n \in \nn^+}$.  
\begin{lem}\label{lem_extension_by_zero_isomorphism_Ban}
	Under Assumption~\ref{ass_growth_condition_sequence}, for each $n \in \nn^+$ and each $p \in [1,\infty)$, 
	\begin{enumerate}[(i)]
		\item The subspace topology on $L^p_{\mu}(K_n)$ is equivalent to the Banach space topology induced by
	$$
	\|f\|_{L^p_{\mu}(K_n)}\triangleq \int_{x \in \rrd} \|f(x)\| I_{K_n}(x) d \mu(x).
	$$
	\item The "extension by zero" map $Z_n(f)\mapsto f I_{K_n}$ is a Bi-Lipschitz (linear) surjection from $L^p_{\mu_n}(\rrd,\rrD)$ to $L^p_{\mu}(K_n)$; where $\mu_n$ is the finite measured defined by its Radon-Nikodym derivative $\frac{d\mu_n}{d\mu}\triangleq I_{K_n}$.  \item $Z^{-1}_n(g)= g$ and in particular $Z_n$ is a homeomorphism.  
\end{enumerate}
\end{lem}
\begin{proof}[{{Proof of Lemma~\ref{lem_extension_by_zero_isomorphism_Ban}}}]
	Fix $n \in \nn^+$.  For any $f,g \in L^p_{\mu_n}(\rrd,\rrD)$ and any $k \in \rr$ 
	$$
	Z_n(f+kg) = (f + kg)I_{K_n} = fI_{K_n} + k(g I_{K_n}) = Z_n(f) + kZ_n(g)
	,
	$$
	therefore $Z_n$ is linear.  Let $\tilde{Z}_n:L^p_{\mu}(K_n) \to L^p_{\mu_n}(\rrd,\rrD)$ be the map taking $g$ to its equivalence class induced by $\mu_n$.  Then
	$$
	\begin{aligned}
	Z_n \circ \tilde{Z}_n(g)= gI_{K_n} \sim g,
	\end{aligned}
	$$
	where $g_1\sim g_2$ denotes the equivalence relation identifying functions which are equal up to a set of null-$\mu_n$ measure.  Likewise, 
	$$
	\tilde{Z}_n\circ Z_n(f) = \tilde{Z}_n(fI_{K_n})=\tilde{Z}_n(f) = f,
	$$
	where we have used the identification of $fI_{K_n}$ with $f$ in $L^p_{\mu_n}(\rrd,\rrD)$.  Therefore $Z_n$ is a bijection.  
	Lastly, note that since any $g,f\in L^p_{\mu}(K_n)$ are $\mu$-essentially supported on $K_n$ and since $\frac{d\mu_n}{d\mu}=1_{K_n}$ then
	\begin{equation}
	\begin{aligned}
	\|g -f\|_{p:n} = &
	\max_{i=1,\dots,n} 
	\left(
	\int_{x \in \rrd} \|
	I_{K_i}g (x) - I_{K_i} f(x)
	\|^p d\mu(x) 
	\right)^{\frac1{p}}
	\\
	= &
	\max_{i=1,\dots,n} 
	\left(
	\int_{x \in \rrd} \|
	g (x) - f(x)
	\|^p I_{K_i}d\mu(x) 
	\right)^{\frac1{p}}
	\\ = &
	\max_{i=1,\dots,n} 
	\left(
	\int_{x \in \rrd} \|
	g (x) - f(x)
	\|^p d\mu_i(x) 
	\right)^{\frac1{p}}
	%
	.
	\end{aligned}
	\label{eq_continuity_proof}
	\end{equation}	
	Therefore, $Z_n$ is a surjective linear isometry of $L^p_{\mu:n}(\rrd,\rrD)$ onto $\bigoplus_{i=1}^n L^p_{\mu:i}(\rr,\rrD)$ equipped with the norm $\|\cdot\|^{(\infty)}_n$.  The conclusions of (i) and (ii) thus follow upon applying Lemma~\ref{lem_lininity_Desciption} (ii). 
	For (iii), note that every Bi-Lipschitz surjection is a homeomorphism, see \citep[page 78]{heinonen2001lectures}.
\end{proof}
\begin{lem}\label{lem_extension}
	Under Assumption~\ref{ass_growth_condition_sequence}, for every $n \in \nn^+$ 
	\begin{enumerate}[(i)]
  		\item The norm $\|\cdot\|_{p:n}$ and the $L^p$-norm, i.e.:
	$
	\|f\|_{L^p_{\mu}}\triangleq \left(\int_{x \in \rrd} \|f(x)\|^p d\mu(x)\right)^{\frac1{p}}
	$
	are equivalent on $L^p_{\mu:n}(\rrd,\rrD)$.  
	\item The topologies induced by either of these norms are equal and make $L^p_{\mu:n}(\rrd,\rrD)$ into a Banach subspace of $L^p_{\mu,\operatorname{loc}}(\rrd,\rrD)$ (when the latter is equipped with $\tau_{L^p_{\mu,\operatorname{loc}}}$).  
	\end{enumerate}
\end{lem}
\begin{proof}
	Assumption~\ref{ass_growth_condition_sequence} guarantees that for each $n,m \in \nn^+$ with $n\neq m$
	$$
	L^p_{\mu}(K_n) \cap L^p_{\mu}(K_m) =\{0\}.
	$$
	Since $L^p_{\mu,\operatorname{loc}}(\rrd,\rrD)$ is a Fr\'{e}chet space containing each $L^p_{\mu}(K_n)$ and since each $L^p_{\mu}(K_n)$ is a Banach sub-space thereof by Lemma~\ref{lem_extension_by_zero_isomorphism_Ban}, then the result follows from Lemma~\ref{lem_lininity_Desciption} because for every $f \in L^p_{\mu:n}(\rrd,\rrD)$
	$$
	\|f\|_n^{(p)} = \left[
	\sum_{i=1}^n \left(\left(
	\int_{x \in \rrd} \|f(x)\|^p I_{K_i}d\mu(x)
	\right)^{\frac1{p}}\right)^{{p}}
	\right]^{\frac1{p}}
	= 
	\left(
	\int_{x \in \rrd} \|f(x)\|^p d\mu(x)
	\right)^{\frac1{p}} = \|f\|_{L^p_{\mu}}
	,
	$$
	since $\operatorname{ess-supp}(\|f\|)\subseteq \bigcup_{i =1}^n K_i$.  
\end{proof}
\subsubsection{Proofs of Results from Section $2$}
\begin{proof}[{Proof of Proposition~\ref{prop_existence}}]
Denote the topology on $L^p_{\mu,\text{strict}}(\rrd,\rrD)$ by $\tau$.  Appealing to \citep[I.2.3, Example 5]{bourbaki2013topological} it is enough to show that $\tau_{\infty}$ exists in order to conclude that $\tau$ exists.  Indeed, by construction, for each $n,m \in \nn^+$ with $n \neq m$, $L^p_{\mu:n}(\rrd,\rrD) \cap L^p_{\mu:m}(\rrd,\rrD)=\{0\}$.  By Lemma~\ref{lem_extension}, for each $n \in \nn^+$, $L^p_{\mu:n}(\rrd,\rrD)$ is a Banach sub-space of $L^p_{\mu,\operatorname{loc}}(\rrd,\rrD)$.  Moreover, since $L^p_{\mu,\operatorname{loc}}(\rrd,\rrD)$ is a Fr\'{e}chet space with metric given by~\eqref{eq_locally_integrable_definition_metric}, therefore the existence of $\tau_{\infty}$ follows directly from Lemma~\ref{lem_details_cocompleteness_in_Top}.  
\end{proof}
\begin{proof}[{Proof of Proposition~\ref{prop_choice_free_construction}}]
By Lemma~\ref{lem_universal_constructions} (i), the underlying sets of $L^p_{\mu:\infty}\left(\rrd,\rrD\right)$ and $\varinjlim_{n \in \nn^+} L_{\mu}^p\left(\bigcup_{i=1}^n K_i\right)$ are the same.  By Lemmas~\ref{lem_extension_by_zero_isomorphism_Ban} (iii) and~\ref{lem_extension} (ii) the map $g_n:f\to f$ form $L^p_{\mu:n}(\rrd,\rrD)$ to $L^p_{\mu}\left(\bigcup_{i=1}^n K_i\right)$ is a homeomorphism.  Therefore, $\left\{\iota_{L^p_{\mu}\left(\bigcup_{i=1}^n K_i\right)}\circ Z_n\right\}_{n \in \nn^+}$ is a compatible system of maps, with the direct system $\left\{L^p_{\mu:n}(\rrd,\rrD),\iota_{n,m}\right\}$, in the sense of Lemma~\ref{lem_universal_constructions} (ii), and in particular the map $\phi:f\to f$ from $L^p_{\mu:\infty}\left(\rrd,\rrD\right)$ to $\varinjlim_{n \in \nn^+} L_{\mu}^p\left(\bigcup_{i=1}^n K_i\right)$ satisfies
	$$
\iota_{L^p_{\mu}\left(\bigcup_{i=1}^n K_i\right)}\circ g_n = \phi\circ j_{L^p_{\mu:n}(\rrd,\rrD)},
	$$
	and satisfies the conclusion of Lemma~\ref{lem_universal_constructions} (ii); where $j_{L^p_{\mu:n}(\rrd,\rrD)}:L^p_{\mu:n}(\rrd,\rrD)\to L^p_{\mu:\infty}(\rrd,\rrD)$ and $\iota_{L^p_{\mu}\left(\bigcup_{i=1}^n K_i\right)}:L^p_{\mu}\left(\bigcup_{i=1}^n K_i\right)\to \varinjlim_{n \in \nn^+} L^p_{\mu}\left(\bigcup_{i=1}^n K_i\right)$ are the inclusion maps.  Since each $g_n$ is a homeomorphism then Lemma~\ref{lem_universal_constructions} (ii) implies that
	\begin{equation}
	\begin{aligned}
	\phi: L^p_{\mu:\infty}(\rrd,\rrD)& \rightarrow \varinjlim_{n \in \nn^+} L^p_{\mu}\left(\bigcup_{i =1}^n K_i\right)\\
	f & \to f
	,
	\end{aligned}
\label{eq_Lp_andmaxLp_are_topologically_identical}
	\end{equation}
	is a homeomorphism.  
	
	Let $I\triangleq \left\{K \subseteq \rrd:\, \mu(K)>0 \mbox{ and $K$ compact}\right\}$.  Make $I$ into a partially ordered set by equipping it with the partial order $K\leq K'$ defined for $K,K'\in I$ as follows: if there exist Borel subsets $Z,Z'\subseteq \rrd$ such that $\mu(Z)=\mu(Z')=0$ and
	$$
	K-Z\subseteq K'- Z'.
	$$
	Since $\{K_n\}_{n \in \nn^+}$ is a partition of $\rrd$ satisfying Assumption~\ref{ass_growth_condition_sequence} then $\mu\left(\bigcup_{i=1}^n K_i\right)> \mu(K_n)>0$ and since the finite union of compact subsets of $\rrd$ is again compact then 	$\{\bigcup_{i=1}^n K_i\}_{n \in \nn^+}\subset I$.  Moreover, for every $K\in I$, there is some $N_K \in \nn^+$ satisfying $K \hyp \left(\rrd- \bigcup_{n \in \nn^+} K_i\right)\subseteq \bigcup_{i=1}^{N_K} K_i$.  By definition of $\left\{K_n\right\}_{n \in \nn^+}$ being a partition of $\rrd$ we have that $\mu\left(
	\rrd- \bigcup_{n \in \nn^+} K_i
	\right)=0$.  Therefore, for every $K \in I$, $K \leq \bigcup_{i=1}^{N_K} K_i$.  Hence, $\{K_n\}_{n \in \nn^+}$ satisfies the requirements of Lemma~\ref{lem_universal_constructions} (iii).  Hence,
	\begin{equation}
	\varinjlim_{n \in \nn^+} L^p_{\mu}\left(\bigcup_{i=1}^n K_i\right) \cong \varinjlim_{K \in I} L^p_{\mu}(K)
\label{eq_proof_cofinality}
.
	\end{equation}
	Combining~\eqref{eq_proof_cofinality} and~\eqref{eq_Lp_andmaxLp_are_topologically_identical} implies that 
	\begin{equation}
L^p_{\mu:\infty}(\rrd,\rrD) \cong \varinjlim_{K \in I} L^p_{\mu}(K)
\label{eq_thrid_homeomorphism}
.
	\end{equation}
By Lemma~\ref{lem_density_inductive_limit_spaces} (i), note that, as a set $\varinjlim_{K \in I} L^p_{\mu}(K) \subseteq L^p_{\mu,\operatorname{loc}}(\rrd,\rrD)$.  Define $L'$ as the topological space with underlying set $L^p_{\mu,\operatorname{loc}}(\rrd,\rrD)$ and equipped with smallest topology containing $\tilde{\tau}\cup \tau_{L^p_{\mu}}\cup \tau_{L^p_{\mu,\operatorname{loc}}}$, where $\tilde{\tau}$ is the topology of $\varinjlim_{K \in I} L^p_{\mu}(K)$.  Thus,\eqref{eq_thrid_homeomorphism} implies that the map $f \to f$ from $L'$ to $L^p_{\mu,\text{strict}}(\rrd,\rrD)$ is a homeomorphism.  Since $\varinjlim_{K \in I} L^p_{\mu}(K)$ is defined independently of the choice of partition $\{K_n\}_{n \in \nn^+}$ satisfying Assumption~\ref{ass_growth_condition_sequence} then this gives the conclusion.
\end{proof}
\begin{proof}[{Proof of Proposition~\ref{prop_Hausdorff}}]
By construction $\tau_{L^p_{\mu,\operatorname{loc}}}\subset \tau$ and since $L^p_{\mu,\operatorname{loc}}(\rrd,\rrD)$ is a Hausdorff space, then for every $f \in L^p_{\mu,\operatorname{loc}}(\rrd,\rrD)$ the set $L^p_{\mu,\operatorname{loc}}(\rrd,\rrD)\hyp\{f\} \in \tau_{L^p_{\mu,\operatorname{loc}}}$.  Therefore, it is an element of $\tau$; whence $L^p_{\mu,\text{strict}}(\rrd,\rrD)$ is Hausdorff.  Since every convergent sequence in a Hausdorff space has a unique limit, then we obtain
\end{proof}
\begin{proof}[{Proof of Proposition~\ref{prop_fineLP_is_fine}}]
    By construction, for each $n,m \in \nn^+$ with $n \neq m$, $L^p_{\mu:n}(\rrd,\rrD) \cap L^p_{\mu:m}(\rrd,\rrD)=\{0\}$.  By Lemma~\ref{lem_extension}, for each $n \in \nn^+$, $L^p_{\mu:n}(\rrd,\rrD)$ is a Banach sub-space of $L^p_{\mu,\operatorname{loc}}(\rrd,\rrD)$.  Moreover, since $L^p_{\mu,\operatorname{loc}}(\rrd,\rrD)$ is a Fr\'{e}chet space with metric given by~\eqref{eq_locally_integrable_definition_metric}, then $\tau_{\infty}$ is strictly finer than $\tau_{L^p_{\mu}}|_{L^p_{\mu:\infty}(\rrd,\rrD)}$ and $\tau_{L^p_{\mu,\operatorname{loc}}}|_{L^p_{\mu:\infty}(\rrd,\rrD)}$, the restriction of the topologies 
	$
	\tau_{L^p_{\mu}}$ and $\tau_{L^p_{\mu,\operatorname{loc}}}
	$
	to the set $L^p_{\mu:\infty}(\rrd,\rrD)$.  	
	Since $\tau_{\infty},\tau_{L^p_{\mu}},\tau_{L^p_{\mu,\operatorname{loc}}}\subseteq \tau$ then $\tau$ is strictly finer than both $\tau_{L^p_{\mu}}$ and $\tau_{L^p_{\mu,\operatorname{loc}}}$.  This gives us the conclusion.  
\end{proof}
\begin{proof}[{Proof of Proposition~\ref{prop_identification}}]
	Fix $f \in L^p_{\mu:n}(\rrd,\rrD)\subset L^p_{\mu,\text{strict}}(\rrd,\rrD)$.  
	Suppose that there exists some sequence $\{f_k\}_{k \in \nn}$ which converges to $f$ in $L^p_{\mu,\text{strict}}(\rrd,\rrD)$.  
	This means that for every $U\in \tau$ containing $g$ there exists some $K \in \nn$ for which $f_K \in U$.  As in the proof of Lemma~\ref{lem_details_cocompleteness_in_Top}, let $\tau'$ denote the finest topology making $L^p_{\mu:\infty}(\rrd,\rrD)$ into a locally-convex space and making each $L^p_{\mu:n}(\rrd,\rrD)$ into a subspace and by the same remarks note that $\tau'\subseteq \tau_{\infty}$.  Since $\tau'$ is coarser than $\tau_{\infty}$, then convergence in $\tau_{\infty}$ implies convergence in $\tau'$.  We work with $\tau'$ due to the availability of certain useful results.  Now, since $f \in L^p_{\mu:\infty}(\rrd,\rrD)$ and since $\tau'\subseteq \tau_{\infty}\subset \tau$ then $\{f_k\}_{k \in \nn}$ must converge to $f$ in $\tau'$.  
		
	If $\{f_k\}_{k \in \nn^+}$ converges to $f$ in $L^p_{\mu:\infty}(\rrd,\rrD)$ with respect to $\tau'$ then it must be eventually bounded.  By \citep[Proposition 4]{dieudonne1949dualite} any bounded subset of $L^p_{\mu:\infty}(\rrd,\rrD)$ must be contained in some $L^p_{\mu:n}(\rrd,\rrD)$.  Therefore, there is some $N_0\in \nn^+$ such that the sequence $\{f_k\}_{k \in \nn^+,\, k \geq N_0}$ is entirely within $L^p_{\mu:n}(\rrd,\rrD)$.  However, by \citep[Proposition 2]{dieudonne1949dualite} the topology $\tau'$ restricted to $L^p_{\mu:n}(\rrd,\rrD)$ coincides with the Banach space topology on $L^p_{\mu:n}(\rrd,\rrD)$, induced by the norm $\|\cdot\|_{p:n}$.  Therefore, the sequence $\{f_k\}_{k \in \nn^+,\, k \geq N_0}$ converges in $L^p_{\mu:n}(\rrd,\rrD)$ for its Banach space topology.
	By definition of $L^p_{\mu:n}(\rrd,\rrD)$, this means that every member of $\{f_k\}_{k \in \nn^+, k \geq N_0}$ satsifies
	$
	\operatorname{ess-supp}(\|f_k\|)\subseteq \bigcup_{i=1}^n K_i
	.
	$
	Hence (iv) holds.  
\end{proof}
\begin{proof}[{Proof of Theorem~\ref{thrm_non_stone_weirestrass}}]
Let $\{K_n\}_{n=1}^{\infty}$ be a family satisfying Assumption~\ref{ass_growth_condition_sequence}; note by Proposition~\ref{prop_choice_free_construction} any (non-)density condition is independent of our choice since (non-)density is preserved by homeomorphisms.  In particular, without loss of generality, let $K_2$ have a non-empty interior.  

Fix some $b\in \rrD-\{0\}$ and $p\in [1,\infty)$.  Suppose that $\fff$ is dense in $L^p_{\mu,\text{fine}}(\rrd,\rrD)$.  Then, there must exists a sequence $\{\hat{f}_n\}_{n=1}^{\infty}$ in $\fff$ converging to $I_{K_n}b \in L^p_{\mu:1}(\rrd,\rrD)$ with respect to the $L^p_{\mu,\text{fine}}(\rrd,\rrD)$ topology.  By Proposition~\ref{prop_identification}, this means that for all but finitely many $f_n$ we have $\operatorname{ess-supp}(f_n)\subseteq K_1$.  Since $K_2$ has non-empty interior then this means that there is an open subset of $\rrd$, namely $\operatorname{int}(K_2)$, on which 
\begin{equation}
    \|f_n(x)\|=0, \qquad (\forall x \in \operatorname{int}(K_n))
    \label{eq_vanishing_condition}
    .
\end{equation}
holds.  However, since each $f_n$ was assumed to be analytic then so is $\|f_n\|\in C(\rrd,\rr)$.  Therefore,~\eqref{eq_vanishing_condition} and the unicity of analytic functions implies that $\|f_n\|=0$ for all but finitely many $n\in \nn_+$; consequentially, the positive-definiteness of $\|\cdot\|$ implies that $f_n=0$ for all but finitely many $n\in \nn_+$.  However this means that:
$$
\int_{x \in \rrd} \|f_n(x)-bI_{K_1}(x)\|d\mu(x)= \|b\|\mu(K_1)>0
$$
for all but finitely many $n$ and therefore $\{f_n\}_{n=1}^{\infty}$ does not converge to $bI_{K_1}$ in $L^p_{\mu,\text{strict}}(\rrd,\rrD)$.  We have arrived at a contradiction, thus, $\fff$ is not dense in $L^p_{\mu,\text{strict}}(\rrd,\rrD)$.  In particular, $\fff$ cannot have the strict $L^p$-UAP.  
\end{proof}
\begin{proof}[{Proof of Corollary~\ref{cor_polynomials_not_dense}}]
    Since every polynomial is analytic then the conditions of Theorem~\ref{thrm_non_stone_weirestrass} are met; whence, the result follows.  
\end{proof}

\subsection{{Proofs of Results for Section~\ref{s_Main_ss_improvements}}}\label{ss_proofs_general}
\subsubsection{{Technical Lemmas for Section~\ref{s_Main_ss_improvements}}}\label{ss_Techincal_Lemmas}
\begin{lem}\label{lem_going_up}
	Let $\emptyset \neq \fff\subseteq  X\subseteq Y$, let $\tau_X$ and $\tau_Y$ be topologies on $X$ and on $Y$, respectively, and let $\tau_Y'$ denote the subspace topology on $X$ induced by restriction of $\tau_Y$.  Denote the smallest topology on $Y$ containing $\tau_X\cup \tau_Y$ by $\tau_X\vee \tau_Y$.  If:
	\begin{enumerate}[(i)]
		\item $\tau_Y'\subseteq \tau_X$,
		\item $\fff$ is dense in $(X,\tau_X)$,
		\item $X$ is dense in $(Y,\tau_Y)$,
	\end{enumerate}
	then $\fff$ is dense in $\tau_X\vee \tau_Y$.  Moreover, if $\tau_X$ is strictly finer than $\tau_Y$, then $\tau_X\vee \tau_Y$ is strictly finer than $\tau_Y$.  
\end{lem}
\begin{proof}
	First note that, since $\fff$ is dense in $X$ with respect to $\tau_X$, and $\tau_Y'\subseteq \tau_X$, then $\fff$ is dense in $(X,\tau_Y')$.  Since density is transitive, and $X$ is dense in $(Y,\tau_Y)$ then $\fff$ is dense in $(Y,\tau_Y)$.  
	Since, $\tau_Y'\subseteq \tau_X$, then the intersection of any $U \in \tau_Y$ and $W\in \tau_X$ satisfies $U\cap W \in \tau_X$.  Therefore, the set $\tau_X\cup \tau_{Y} $ is closed under finite intersection.  Hence, every $U \in \tau_X\vee \tau_Y$ must be of the form
	$$
	U = \bigcup_{i \in I_1} U_{i,1} \cup \bigcup_{i \in I_2} U_{i,2},
	$$
	for some indexing sets $I_1$ and $I_2$, and some subsets $U_{i,1},U_{i,2}$ contained in $\tau_X$ and in $\tau_Y$, respectively.  Assume that $I_1$ and $I_2$ are non-empty or else there is nothing to show.  Since $\fff$ is dense in $(Y,\tau_Y)$ and $(X,\tau_X)$ then there exist $f_1,f_2 \in \fff$ such that
	$$
	f_1 \in \bigcup_{i \in I_1} U_{i,1} \mbox{ and } f_2 \in \bigcup_{i \in I_2} U_{i,2}.
	$$
	Therefore, $\fff \cap \bigcup_{i \in I_1} U_{i,1} \cup \bigcup_{i \in I_2} U_{i,2}$ is non-empty.  Whence, $\fff$ is dense in $(Y,\tau_X\vee \tau_Y)$.  
\end{proof}
\begin{lem}\label{lem_density_inductive_limit_spaces}
	Let $\{X_n\}_{n \in \nn^+}$ be Banach subspaces of a Fr\'{e}chet space $X$, for which $X_n \cap X_m = \{0\}$ for each $n\neq m$, $n,m \in \nn^+$.  Suppose that, for each $n \in \nn^+$, $D_n\subseteq X_n$ is dense in $X_n$ for its Banach space topology.  Then
	$$
	\bigcup_{n \in \nn^+} \left\{x \in \bigoplus_{i=1}^n X_i:\, x = \sum_{i=1}^{n}\beta_i x_i ,\, \beta_i \in \rr,\, x_i \in D_i\right\},
	$$
	is dense in $\bigcup_{i\in \nn} \overline{ \overset{n}{\underset{i=1}{\bigoplus}} D_i}$ when it is equipped with the topology of Lemma~\ref{lem_details_cocompleteness_in_Top}.  
\end{lem}
\begin{proof}
	For each $n \in \nn$, note that the set $ \left\{x \in \oplus_{i=1}^n X_i:\, x = \sum_{i=1}^{n}\beta_i x_i ,\, \beta_i \in \rr,\, x_i \in D_i\right\}$ is precisely the span of $\{D_i\}_{i=1}^n$; thus it is denoted by $\operatorname{span}(\{D_i\}_{i=1}^n)$.  Moreover, let $\tau$ denote the topology on $\bigcup_{n \in \nn} \overline{\overset{n}{\underset{i=1}{\bigoplus}} X_i}$ defined in Lemma~\ref{lem_details_cocompleteness_in_Top}.

	First, we show that $\operatorname{span}(\{D_i\}_{i=1}^n)$ is dense in $\overline{ \overset{n}{\underset{i=1}{\bigoplus}} X_i}$ with respect to the topology induced by the norm $\|\cdot\|_n'$.  
	By Lemma~\ref{lem_lininity_Desciption} (iii) then
	$\overset{n}{\underset{i=1}{\bigoplus}} X_i$ is dense in its completion, which is equal to $\overline{ \overset{n}{\underset{i=1}{\bigoplus}} X_i}$.  Therefore, since density is transitive then it is sufficient to show that $\operatorname{span}(\{D_i\}_{i=1}^n)$ is dense in $\overset{n}{\underset{i=1}{\bigoplus}} X_i$ to conclude that it is dense in $\overline{\overset{n}{\underset{i=1}{\bigoplus}} X_i}$.

	Consider the case where $n=1$.  Since $X_1$ is a Banach subspace of $X$ then $X_1=\overline{\operatorname{span}(X_1)}$.  Since $D_1$ is dense in $X_1$ and since $D_1\subseteq \operatorname{span}(D_1)\subseteq X_1\subseteq \overline{\operatorname{span}(X_1)}$, then $D_1$ is dense in $\overline{\operatorname{span}(X_1)}$.  
	
	Next, suppose that $n \in \nn^+$ and $n> 1$.  For every $i \in \nn$, $D_i$ is dense in the Banach space $X_i$, therefore, for every $\epsilon>0$ and every ${x}_i \in X_i$ there exists some $d_{i,\epsilon} \in D_i$ satisfying
	\begin{equation}
	\left\|
	d_{i,\epsilon} - {x}_i
	\right\|_{(i)}<\frac{\epsilon}{n}
	\label{eq_proof_lem_bonding_dense_subsets_in_construction_of_Lemma_refinement_lemma_label_density_in_Banach}
	,
	\end{equation}	
	where $\|\cdot\|_{(i)}$ is the norm on $X_i$.  
	For every $i=1,\dots,n$ let $d_{i,\epsilon} \in D_i$ be such that~\eqref{eq_proof_lem_bonding_dense_subsets_in_construction_of_Lemma_refinement_lemma_label_density_in_Banach} holds and note that $\sum_{i=1}^{n} d_i \in \operatorname{span}(\{D_i\}_{i=1}^n)$, $d_i-x_i \in X_i$, and that
	\begin{equation}
	\sum_{i=1}^nd_i-\sum_{i=1}^nx_i =
	\sum_{i=1}^n(d_i-x_i) \in \overset{n}{\underset{i=1}{\bigoplus}} X_i
	\label{eq_proof_lem_bonding_dense_subsets_in_construction_of_Lemma_refinement_lemma_containement_noticing}
	.
	\end{equation}

	Since the norm on $\overset{n}{\underset{i=1}{\bigoplus}} X_i$ is given by~\eqref{eq_an_interpolation_norm} then by definition of the infimum,~\eqref{eq_proof_lem_bonding_dense_subsets_in_construction_of_Lemma_refinement_lemma_containement_noticing} and by~\eqref{eq_proof_lem_bonding_dense_subsets_in_construction_of_Lemma_refinement_lemma_label_density_in_Banach} we have that
	\begin{equation}
	\begin{aligned}
	\left\|
	\sum_{i=1}^{n} d_i
	-
	\sum_{i=1}^{n} x_i
	\right\|_n'
	& \leq
	\sum_{i=1}^{n}
	\left\|
	d_i
	-
	x_i
	\right\|_i
	& \leq \sum_{i=1}^{n} \frac{\epsilon }{n} = \epsilon
	.
	\end{aligned}
	\label{eq_proof_lem_bonding_dense_subsets_in_construction_of_Lemma_refinement_obtaining_density_on_span}
	\end{equation}
	Therefore, $\operatorname{span}({D_i}_{i=1}^n)$ is dense in $\overset{n}{\underset{i=1}{\bigoplus}} X_i$.  Consequently, it is dense in $\overline{\overset{n}{\underset{i=1}{\bigoplus}} X_i}$.  Since $D_1+\dots+D_n\subseteq \operatorname{span}(\{D_i\}_{i=1}^n)$ and since $D_1+\dots+D_n$ is dense in $\bigcup_{n \in \nn} \overline{\overset{n}{\underset{i=1}{\bigoplus}} X_i}$ with respect to $\tau$, then so is $\operatorname{span}(\{D_i\}_{i=1}^n)$.  In particular, it is dense in $\overline{\operatorname{span}(\{D_i\}_{i=1}^n)}$.

	Next, we show that $\bigcup_{n \in \nn} \operatorname{span}(\{D_i\}_{i=1}^n)$ is dense in $\bigcup_{n \in \nn} \overline{\overset{n}{\underset{i=1}{\bigoplus}} X_i}$ with respect to $\tau$.  Let $x \in \bigcup_{n \in \nn} \overline{\overset{n}{\underset{i=1}{\bigoplus}} X_i}$ and let $U_x$ be open set in $\tau$ containing $x$.  By construction, there must exist an $N\in \nn$ such that 
	$
	x \in 
	\overline{\overset{N}{\underset{i=1}{\bigoplus}} X_i}
	.
	$
	By the continuity of the quotient map $q: \coprod_{n \in \nn} \overline{\overset{n}{\underset{i=1}{\bigoplus}} X_i}
	\rightarrow 
	\bigcup_{n \in \nn} \overline{\overset{n}{\underset{i=1}{\bigoplus}} X_i}
	,
	$
	the set $q^{-1}[U_x]$ is non-empty and open in $\coprod_{n\in \nn} \overline{
		\overset{n}{\underset{i=1}{\bigoplus}} X_i	
	}$, thus, $q^{-1}[U_x]\cap 
	\overline{
		\overset{N}{\underset{i=1}{\bigoplus}} X_i	
	}
	$ is an open subset of $
	\overline{
		\overset{N}{\underset{i=1}{\bigoplus}} X_i
	}
	$.  
	Since $\operatorname{span}(\{D_i\}_{i=1}^N)$ is dense in $\overline{
		\overset{N}{\underset{i=1}{\bigoplus}} X_i
	}$, then there exists some $d \in \bigcup_{n \in \nn} \operatorname{span}(\{D_i\}_{i=1}^n)$ such that $d$ lies in $q^{-1}[U_x]\cap \overline{
		\overset{N}{\underset{i=1}{\bigoplus}} X_i
	}$.  Therefore,
	$
	\emptyset
	\neq 
	U_x \cap \operatorname{span}(\{D_i\}_{i=1}^N)
	\subseteq 
	U_x \cap \bigcup_{n \in \nn} \operatorname{span}(\{D_i\}_{i=1}^N)
	.
	$
	Hence, $\bigcup_{n \in \nn} \operatorname{span}(\{D_i\}_{i=1}^n)$ is dense in $\bigcup_{n \in \nn} \overline{
		\overset{n}{\underset{i=1}{\bigoplus}} X_i	
	}$.
\end{proof}
\subsubsection{Proofs of Main Results}

Let $\fff$ be a non-empty subset of $C(\rrd;\rrD)$.  The following subset of $\tope{\fff}$ will play an important role in many of the following proofs
\begin{equation}
\begin{aligned}
\tope{\fff}[+]
\triangleq &
\left\{
f \in \tope{\fff}:\, 
(\exists \beta_1,\dots,\beta_n \in \rr)(\exists f_1,\dots,f_n \in \fff)\,
f =
\sum_{i=1}^{n} \beta_i I_i f_i
\right\}
.
\end{aligned}
\end{equation}
In other words, $\tope{\fff}[+]=\tope{\fff}\cap L^p_{\mu:\infty}(\rrd,\rrD)$.  For many of the optimization results, it is necessary to consider all of $\tope{\fff}$, however, the next universal approximation result is entirely due to the structure of $\tope{\fff}[+]$ and the set defined by the difference $\tope{\fff}\hyp\tope{\fff}[+]$ is not required.  
%
\begin{proof}[Proof of Theorem~\ref{thrm_bagged_localization_LP}] %
	Since Assumption~\ref{ass_growth_condition_sequence} holds, then by Lemma~\ref{lem_extension_by_zero_isomorphism_Ban}, for every $n\in \nn^+$, the \textit{extension by zero-map} $Z_n$ defines a continuous homeomorphism from $L^p_{\mu_n}(\rrd,\rrD)$ onto $L^p_{\mu}(K_n)$; where $\mu_n$ is defined in Lemma~\ref{lem_extension_by_zero_isomorphism_Ban}.  Since $\fff$ was assumed to be dense in $L^p_{\nu}(\rrd;\rrD)$ for every finite Borel measure on $\rrd$ dominated by the Lebesgue measure, and since $\mu_n$ is such a measure then $\fff$ is dense in $L^p_{\mu_{n}}(\rrd;\rrD)$, for every $n \in \nn$.  
	Therefore, by assumption, the set $\fff$ must be dense in $L^p_{\mu_n}(\rrd;\rrD)$.  Since $Z_n$ is a surjective Bi-Lipschitz map it is in particular a continuous surjection.  Thus, for each $n \in \nn^+$,
	$$
	Z_n(\fff)= \left\{g \in L^p_{\mu}(\rrd;\rrD):\, (\exists f \in\fff)\, \mbox{ s.t. } g= I_n f \quad \mu-a.e.\right\},
	$$
	is dense in $L^p_{\mu:n}(\rrd;\rrD)$.  Moreover, since $L^p_{\mu}(K_n)\cap L^p_{\mu}(K_m) =\{0\}$ for $n,m \in \nn^+$ and $n\neq m$ 
	then by Lemma~\ref{lem_density_inductive_limit_spaces} 
	$
	\tope{\fff}[+]=\bigcup_{n \in \nn^+} Z_n(\fff)$ is dense in $\bigcup_{n\in \nn^+} L^p_{\mu:n}(\rrd;\rrD)= L^p_{\mu,\text{strict}}(\rrd,\rrD)$ with respect to $\tau$.

	By Proposition~\ref{prop_fineLP_is_fine}, the topology on $L^p_{\mu:\infty}(\rrd,\rrD)$ is strictly finer than both $\tau_{L^p_{\mu,\operatorname{loc}}}$ and $\tau_{L^p_{\mu}}$ when restricted to the set $L^p_{\mu:\infty}(\rrd,\rrD)$.  Moreover, since $L^p_{\mu:\infty}(\rrd,\rrD)$ contains all compactly-supported simple functions then it forms a dense subset of $L^p_{\mu}(\rrd,\rrD)$.  
	Applying Lemma~\ref{lem_going_up} we see that $\tope{\fff}$ is dense in $L^p_{\mu}(\rrd,\rrD)$ with respect to $\tau_{\infty}\vee \tau_{L^p_{\mu}}$, the smallest topology containing $\tau_{\infty}\cup \tau_{L^p_{\mu}}$.  
	
Since $\tau_{L^p_{\mu,\operatorname{loc}}}$ restricted to $L^p_{\mu}(\rrd,\rrD)$ is coarser than $\tau_{L^p_{\mu}}$ it must be coarser than $\tau_{\infty}\vee \tau_{L^p_{\mu}}$.  	Since $L^p_{\mu}(\rrd,\rrD)$ is dense in $L^p_{\mu,\operatorname{loc}}(\rrd,\rrD)$ with respect to $\tau_{L^p_{\mu,\operatorname{loc}}}$ then applying Lemma~\ref{lem_going_up} again we find that $\tope{\fff}$ is dense in $(\tau_{\infty}\vee \tau_{L^p_{\mu}})\vee \tau_{L^p_{\mu,\operatorname{loc}}}$, the smallest topology containing 
\begin{equation}
(\tau_{\infty} \cup \tau_{L^p_{\mu}})\cup \tau_{L^p_{\mu,\operatorname{loc}}} 
= 
\tau_{\infty} \cup \tau_{L^p_{\mu}}\cup \tau_{L^p_{\mu,\operatorname{loc}}}
\label{eq_simplgy_thanks}
.
\end{equation}
Since the right-hand side of~\eqref{eq_simplgy_thanks} is precisely the definition of the topology on $L^p_{\mu,\text{strict}}(\rrd,\rrD)$ then the first conclusion follows; i.e. $\tope{\fff}$ is dense in $L^p_{\mu,\text{strict}}(\rrd,\rrD)$.  Since $\mu$ was an arbitrary $\sigma$-finite Borel measure on $\rrd$ which is absolutely continuous with respect to the Lebesgue measure on $\rrd$ then the conclusion held for all such measures.  Hence, $\tope{\fff}$ has the strict $L^p$-UAP.  
\end{proof}
\begin{proof}[Proof of {Theorem~\ref{thrm_Gain}}]
	For each $f \in \fff$, notice that by taking $\beta_0,\dots,\beta_n =1$ and $f_0,\dots,f_n=f$ in (5) implies that $\fff\subseteq \tope{\fff}$.  Therefore, for any topology $\tau$ on $X$, $\overline{\fff}\subseteq \overline{\tope{\fff}}$, where $\overline{A}$ denote the closure of a subset $A\subseteq X$ with respect to $\tau$.  This gives (i).  
	
	For (ii), fix $f \in \tope{\fff}-\fff$, the set of elements in $\tope{\fff}$ which are not in $\fff$, and set $\tau\triangleq \left\{
	X,\emptyset, \{f\}
	\right\}$.  Since the union and intersection of any pair of sets in $
	\tau$ is again in $\tau$ and since $\tau$ contains $X$ and $\emptyset$ then it is indeed a topology on $X$.  Since $\tau$ strictly contains $\{X,\emptyset\}$ then it is non-trivial.  Since $f \in X$ and $f \in \{f\}$ and by construction $f \in \tope{\fff}$ then $\tope{\fff}$ is dense in $X$ for $\tau$.  Since $f \not \in \fff$ then $\fff \cap \{f\}= \emptyset$.  Therefore, $\fff$ cannot be dense in $X$ with respect to $\tau$.  
\end{proof}
\begin{proof}[{Proof of Corollary~\ref{cor_upgraded_UAT}}]
Since, for every $\sigma$-finite Borel measure $\mu$ on $\rrd$ which is absolutely continuous with respect to the Lebesgue measure, $\tope{\fff}$ is dense in $L^p_{\mu,\text{fine}}(\rrd,\rrD)$ if $\fff$ has the local $L^p$-UAP then Proposition~\ref{prop_fineLP_is_fine} implies that $\tope{\fff}\cap L^p_{\mu,\infty}(\rrd,\rrD)$ is dense in $L^p_{\mu,\infty}(\rrd,\rrD)$ for the $\tau_{\mu,\infty}$ topology which is, by construction, at-least as fine as the norm-topology on $L^p_{\mu}(\rrd,\rrD)$.  Now, since $L^p_{\mu,\infty}(\rrd,\rrD)$ is dense in $L^p_{\mu}(\rrd,\rrD)$ for the norm topology and since density is transitive, then $\tope{\fff}\cap L^p_{\mu,\infty}(\rrd,\rrD)$ is dense in $L^p_{\mu}(\rrd,\rrD)$.  Since this argument, help independently of the choice of $\sigma$-finite Borel measure $\mu$ which was absolutely continuous with respect to the Lebesgue measure, then $\tope{\fff}$ has the global $L^p$-UAP.  
\end{proof}
\subsection{{Proof of Application Section~\ref{sss_applications}}}\label{s_proof_applications}
\begin{proof}[{Proof of Corollary~\ref{cor_ffNN_case}}]
	Let $\fff$ denote the set of all feedforward neural networks from $\rrd$ to $\rrD$ with one hidden layer and with activation function $\sigma$.  
	Since $\sigma$ is continuous, non-polynomial, and bounded then \citep[Proposition 1]{kidger2020universal} implies that $\fff$ is dense in $L^p_{\nu}(\rrd;\rrD)$ for every compactly-supported finite Borel measure $\nu$ on $\rrd$.  
	Thus, the result follows from Theorems~\ref{thrm_bagged_localization_LP} and~\ref{thrm_Gain}.  
\end{proof}
\begin{proof}[Proof of {Corollary~\ref{cor_cnn_case}}]
If $\nu$ is trivial then there is nothing to show.  Therefore, assume that $\nu$ is non-trivial.  
	Let $\nu$ be any finite measure on $(\rrd,\Sigma)$, where $\Sigma$ is the Borel $\sigma$-algebra on $\rrd$, and suppose that $\nu$ is supported on some non-empty compact subset $K\subset \rrd$.  We first show that $\operatorname{Conv}^{s}$ is dense in $L^p_{\nu}(\rrd)$. 
	For any $r>0$, let $B_{r}(0)$ denote the closed unit ball in $\rrd$ centered at $0$ of radius $r$ and let $1_{B_r(0)}$ denote the indicator function of that set.  
	
	Since every finite measure is a multiple of a probability measure then without loss of generality we may assume that $\nu$ is a Borel probability measure on $\rrd$.  Since $\rrd$ is a Polish space then by \citep[Theorem 13.6]{klenke2013probability} $\nu$ is a Radon measure.  By \citep[Theorem  3.14]{rudin2006real}, consequence of Lusin's theorem, this implies that the set of continuous compactly-supported functions from $\rrd$ to $\rr$ are dense in $L^p_{\nu}(\rrd)$.  Therefore, for every $f \in L_{\nu}^p(\rrd)$ and $\delta>0$ there exists some continuous and compactly supported $\tilde{f} \in C(\rrd,\rr)$ satisfying
	\begin{equation}
	\int_{x \in \rrd} |f(x)-\tilde{f}(x)|^p d\nu(x) < \frac{\delta^p}{2^p}
	\label{eq_LP_density_bound1_approx_by_Cc_functions}
	.
	\end{equation}
	Since $\nu$ is a non-trivial finite measure then $0< \nu(\rrd)<\infty$ and $0<\sqrt[p]{\frac{\delta^p}{2^p(1 + \nu(\rrd))}}<\infty$.  
	In particular, $1_{\rrd}$ is $\mu$-integrable since 
	$\int_{x \in \rrd} 1_{\rrd}(x)d\mu(x) = \nu(\rrd)<\infty$.   
	
	Since $\tilde{f}$, in~\eqref{eq_LP_density_bound1_approx_by_Cc_functions}, is compactly supported then there exists a compact subset $K\subset \rrd$ satisfying $\operatorname{supp}(f)\subseteq K$.  By \citep[Theorem 1]{UniversalDeepConv} $\operatorname{Conv}^{s}$ is dense in $C(\rrd,\rr)$ for the topology of uniform convergence on compacts and since $\tilde{f}$ is compactly supported then there exists some $f_{\delta} \in \operatorname{Conv}^s$ such that
	\begin{equation}
	\sup_{x \in \rrd} \left|
	f_{\delta}(x)
	- \tilde{f}(x)
	\right|
	< 
	\frac{\delta^p}{2^p(1 + \nu(\rrd))}
	\label{eq_LP_density_bound2_convolutional_universality_bound_sup_form}
	.
	\end{equation}
	Note that the right-hand side of~\eqref{eq_LP_density_bound2_convolutional_universality_bound_sup_form} is bounded above by $\frac{\delta^p}{2^p}$.  Thus,~\eqref{eq_LP_density_bound2_convolutional_universality_bound_sup_form} implies that
	\begin{equation}
	\begin{aligned}
	\sqrt[p]{\int_{x \in \rrd} \left|
	f_{\delta}(x)
	- \tilde{f}(x)
	\right|^p d\nu(x)
	}
	\leq &
	\sqrt[p]{\frac{\delta^p}{2^p(1 + \nu(\rrd))} \int_{x \in \rrd} 1_{\rrd}(x) d\nu(x)
	}
	\\
	< &
	\sqrt[p]{
	\frac{\delta^p \nu(\rrd)}{2^p(1 + \nu(\rrd))}
	}
	\\<& \frac{\delta}{2}
	.
	\end{aligned}
	\label{eq_LP_density_bound3_convolutional_universality_bound}
	\end{equation}
	Combining~\eqref{eq_LP_density_bound1_approx_by_Cc_functions} and~\eqref{eq_LP_density_bound3_convolutional_universality_bound} with the triangle inequality yields
	$$
	\begin{aligned}
	\sqrt[p]{\int_{x \in \rrd}\left|
	    f(x) - f_{\delta}(x) 
	\right|^p
	d\mu(x)}
	\leq &
	\sqrt[p]{\int_{x \in \rrd}\left|
	    f(x) - \tilde{f}(x) 
	\right|^p
	d\mu(x)}
	+
	\sqrt[p]{\int_{x \in \rrd}\left|
	    \tilde{f}(x) - f_{\delta}(x) 
	\right|^p
	d\mu(x)}\\
	< & 
	\delta
	.
	\end{aligned}
	$$
	Therefore, $\operatorname{Conv}^s$ is dense in $L^p_{\nu}(\rrd)$ for every $p \in [1,\infty)$ and every finite compactly supported Borel measure $\nu$ on $\rrd$ which is absolutely continuous with respect to the Lebesgue measure on $\rrd$.  Therefore, Theorems~\ref{thrm_bagged_localization_LP} and~\ref{thrm_Gain} apply, hence, the conclusion follows.  
	\end{proof}
	\begin{proof}[{Proof of Corollary~\ref{cor_polynomials_dense}}]
    The result directly follows from Theorem~\ref{thrm_bagged_localization_LP} and Theorem~\ref{thrm_non_stone_weirestrass}.  
    \end{proof}
    \begin{proof}[{Proof of Corollary~\ref{cor_polynomials_dense_gen}}]
    The result directly follows from Corollary~\ref{cor_polynomials_dense_gen} since polynomials are analytic.  
    \end{proof}
\end{appendices}

\bibliographystyle{abbrvnat}
\bibliography{Refs}

\end{document}